\documentclass{article}
\usepackage{ijcai21}

\RequirePackage[T1]{fontenc} 
\RequirePackage[tt=false, type1=true]{libertine} 
\RequirePackage[varqu]{zi4} 

\RequirePackage[libertine]{newtxmath}
\usepackage[letterspace=-25]{microtype}

\usepackage{soul}
\usepackage{url}
\usepackage[hidelinks]{hyperref}
\usepackage[utf8]{inputenc}
\usepackage[small]{caption}
\usepackage{graphicx}
\usepackage{amsmath}
\usepackage{booktabs}
\usepackage{savesym}
\savesymbol{Bbbk}
\savesymbol{openbox}

\usepackage{amssymb}
\usepackage{amsthm}
\usepackage{MnSymbol}
\usepackage{color}
\usepackage{url}
\usepackage{multirow}
\usepackage{xcolor}
\usepackage{dsfont}
\usepackage{algorithm}
\usepackage{algpseudocode}

\restoresymbol{NEW}{Bbbk}
\restoresymbol{NEW}{openbox}

\usepackage[switch]{lineno}

\usepackage{tablefootnote}

\usepackage{subcaption}

\newcommand{\FTx}[1]{{#1}} 
\definecolor{cadmiumgreen}{rgb}{0.0, 0.75, 0.40}
\definecolor{richcarmine}{rgb}{0.85,0.00,0.23}
\newcommand{\FT}[1]{\textcolor{black}{#1}} 
\newcommand{\AR}[1]{\textcolor{black}{#1}}
\newcommand{\ARR}[1]{\textcolor{black}{#1}}
\newcommand{\ARX}[1]{{#1}}

\newcommand{\PBX}[1]{{#1}}
\newcommand{\EA}[1]{\textcolor{black}{#1}}
\newcommand{\EAB}[1]{{#1}}
\newcommand{\EAX}[1]{{#1}}
\newcommand{\todo}[1]{\textcolor{magenta}{$*$ #1}}
\newcommand{\delold}[1]{}
\newcommand{\del}[1]{\textcolor{olive}{$\times$ #1}}
\newcommand{\dell}[1]{}

\def\wrt{wrt}
\def\resp{resp.}

\def\code#1{\texttt{#1}}

\newcommand{\plusSTAT}{\ensuremath{\stackrel{\mbox{\scalebox{1}{.}}}{+}}}
\newcommand{\minusSTAT}{\ensuremath{\stackrel{\mbox{\scalebox{1}{.}}}{-}}}

\newcommand{\FFAll}{\ensuremath{\mathbb{F}}}

\newcommand{\Ione}{\ensuremath{\mathbb{I}}}

\newcommand{\Args}{\ensuremath{\mathcal{X}}}
\newcommand{\RV}{\ensuremath{\Args_r}}
\newcommand{\Influences}{\ensuremath{\mathcal{I}}}
\newcommand{\Inputs}{\ensuremath{\mathcal{A}}}

\newcommand{\Rels}{\ensuremath{\mathcal{R}}}
\newcommand{\Atts}{\ensuremath{\Rels_-}}
\newcommand{\Supps}{\ensuremath{\Rels_+}}
\newcommand{\SDAtts}{\ensuremath{\Rels_{\minusSTAT}}}
\newcommand{\SDSupps}{\ensuremath{\Rels_{\plusSTAT}}}
\newcommand{\CIRels}{\ensuremath{\Rels_!}}
\newcommand{\PIRels}{\ensuremath{\Rels_*}}

\newcommand{\SF}{\ensuremath{\sigma}}

\newcommand{\prop}{\ensuremath{{\pi}}}

\newcommand{\Obs}{\ensuremath{\mathcal{O}}}
\newcommand{\Cla}{\ensuremath{\mathcal{C}}}
\newcommand{\Dep}{\ensuremath{\mathcal{D}}}
\newcommand{\Val}{\ensuremath{\mathcal{V}}}

\newcommand{\expvar}{\ensuremath{e}}
\newcommand{\arga}{\ensuremath{x}}
\newcommand{\argb}{\ensuremath{y}}
\newcommand{\argc}{\ensuremath{z}}

\newcommand{\tuple}[1]{\ensuremath{\langle #1 \rangle}}

\newcommand{\set}[1]{\ensuremath{\{ #1 \}}}

\newcommand{\plusLIME}{\ensuremath{\stackrel{\mbox{\scalebox{.30}{LIME}}}{+}}}
\newcommand{\minusLIME}{\ensuremath{\stackrel{\mbox{\scalebox{.30}{LIME}}}{-}}}
\newcommand{\plusSHAP}{\ensuremath{\stackrel{\mbox{\scalebox{.30}{SHAP}}}{+}}}
\newcommand{\minusSHAP}{\ensuremath{\stackrel{\mbox{\scalebox{.30}{SHAP}}}{-}}}
\newcommand{\limefun}{\ensuremath{v_{L\!I\!M\!E}}}

\newcommand{\AttsSHAP}{\ensuremath{\Rels_{\minusSHAP}}}
\newcommand{\SuppsSHAP}{\ensuremath{\Rels_{\plusSHAP}}}
\newcommand{\AttsLIME}{\ensuremath{\Rels_{\minusLIME}}}
\newcommand{\SuppsLIME}{\ensuremath{\Rels_{\plusLIME}}}

\newtheorem{theorem}{Theorem}

\newtheorem{definition}{Definition}
\newtheorem{proposition}{Proposition}
\newtheorem{property}{Property}

\usepackage{fancyhdr}
\fancypagestyle{firstpage}{%
    \fancyhf{}
    \setlength{\headsep}{0.3in}
    \cfoot{\thepage}
}
\fancypagestyle{otherpages}{%
    \fancyhf{}
    \setlength{\headsep}{0.3in}
    \rhead{Influence-Driven Explanations for Bayesian Network Classifiers}
    \lhead{A. Rago, E. Albini, P. Baroni and F. Toni}
    \cfoot{\thepage}
}

\title{\EAB{Influence}-Driven Explanations\\for Bayesian Network Classifiers\\
}

\author {
        Antonio Rago$^1$,
        Emanuele Albini$^1$,
        Pietro Baroni$^2$ and
        Francesca Toni$^1$\\
\affiliations\large
    $^1$ Dept. of Computing, Imperial College London, UK \\
    $^2$ Dip.to di Ingegneria dell'Informazione, Universit\`{a} degli Studi di Brescia, Italy\\
\emails
    \{a.rago, emanuele, ft\}@imperial.ac.uk,
pietro.baroni@unibs.it
}

\begin{document}
\pagestyle{firstpage}
\maketitle

\begin{abstract}
	One of the most pressing issues in AI in recent years has been the need to address the lack of explainability of many of its models. 
	We focus on 
	explanations for \AR{discrete} Bayesian network classifiers (BCs), targeting greater transparency of their inner workings by including \emph{intermediate} variables in explanations, rather than just the input and output variables as is standard practice
	.
	The proposed \emph{\EAB{influence}-driven explanations} (IDXs) for BCs 
	are systematically generated using
	\ARX{the causal relationships between variables \emph{within the BC}, called influences, which are then categorised by logical requirements, called \emph{relation properties}, according to their behaviour.}
	\ARX{These relation}
	properties 
	\ARX{both provide} 
	guarantees 
	\FT{beyond} heuristic explanation methods 
	\ARX{and} allow the information underpinning an explanation to be tailored to a particular context's and user's requirements, e.g., IDXs may be \emph{dialectical} or \emph{counterfactual}
	. 
	We demonstrate IDXs' capability to explain various forms of BCs, e.g., naive or multi-label, binary or categorical, and also integrate recent approaches to explanations for BCs from the literature.
	We evaluate IDXs with theoretical \ARX{and} empirical 
	analyses, demonstrating their considerable advantages when compared with existing explanation methods.
\end{abstract}

\section{Introduction}\label{sec:intro}

The need for explainability has been one of the fastest growing concerns in AI 
\FT{of late}, driven by academia, industry and governments
, as various stakeholders become more conscious of the dangers posed by the widespread implementation of systems we do not fully understand. In response to this pressure, a plethora of explanations for AI methods have been proposed
, with 
\FT{diverse strengths and weaknesses}. 

Some 
approaches to Explainable AI are heuristic and model-agnostic
, e.g., \cite{Ribeiro_16,Lundberg_17}. \FT{Although model-agnosticism leads to broad} 
applicability and uniformity across contexts\FT{, the heuristic nature of these approaches 
 leads to some}
weaknesses\FT{, especially concerning user trust}
\cite{Ignatiev_20}. 
Another limitation, brought about in part by the model-agnosticism, is the fact that these explanations are restricted to representations of how inputs influence outputs, neglecting the influences of intermediate components of models.
This can lead to the undesirable situation where 
\FT{explanations of} the outputs from two \FT{very} diverse systems, e.g., a Bayesian network classifier (BC) and a neural model, performing the same task 
may be identical, despite 
completely different \FT{underpinnings}.
 This trend towards explanations focusing on inputs and outputs exclusively is not limited to model-agnostic explanations, however. 
 Explanations tailored towards specific AI methods are often \FT{also} restricted to inputs' influence on outputs, 
 e.g., the methods of \cite{Shih_18} for BCs or 
 of \cite{Bach_15} for neural networks.
Various methods have been devised for interpreting the intermediate components of neural networks (e.g., see \cite{Bau_17}), and \cite{Olah_18} have shown the benefits of identifying relations between these components and inputs/outputs.
\FT{Some methods exist for accommodating intermediate components in counterfactual explanations (CFXs) of BCs \cite{BC}.}

We propose the novel formalism of \emph{\EAB{influence}-driven explanations} (IDXs) for systematically providing various forms of explanations for a variety of BCs\FT{, and admitting CFXs as instances, amongst others}.
\ARX{The influences provide insight into the causal relationships between variables \emph{within the BC}, which are then categorised by logical requirements, called, \emph{relation properties}, according to their behaviour.}
The \ARX{relation} properties 
provide formal guarantees on the 
\FT{inclusion of \ARX{the required types of} influences \ARX{included} in IDXs.} 
IDXs 
\FT{are} \ARX{thus} fully customisable to the explanatory requirements of a particular application.
\FT{Our} contribution is threefold: \FT{we give}
    \AR{(1)} a systematic \ARR{approach} for \ARR{generating} IDXs from BCs, generalising existing work and offering 
    \FT{great} flexibility with regards to the BC model being explained and the nature of the 
    explanation;
    \AR{(2)} various instantiations of IDXs, including two based on the cardinal principle of \ARX{dialectical} monotonicity and one capturing \FT{(and extending) CFXs}
    ; and
    \AR{(3)}  theoretical \ARX{and} empirical 
    analyses,
    \ARX{showing the strengths of IDXs with respect to existing methods, along with illustrations of real world cases where the exploitation of these benefits may be particularly advantageous}.

\newpage
\section{Related Work}
\pagestyle{otherpages}
\label{sec:related} 

There are a multitude of methods in the literature for providing explanations \FT{(e.g., see the recent survey 
undertaken by \cite{guidotti})}.
Many 
are model-agnostic,
\FT{including:} \emph{attribution} methods such as \emph{LIME} \cite{Ribeiro_16} 
\FT{and} \emph{SHAP} \cite{Lundberg_17}, which assign each feature an \emph{attribution value} indicating 
\FT{its} contribution towards a prediction\FT{; \emph{CXPlain} \cite{Schwab_19}, using a causal model for generating attributions; and methods giving \emph{counterfactual explanations}, such as \emph{CERTIFAI} \cite{Sharma_20}, using a custom genetic algorithm, \emph{CLEAR} \cite{White_20}, using a local regression model, and \emph{FACE} \cite{Poyiadzi_20}, giving actionable explanations}.
\FT{Some of the model-agnostic methods rely upon symbolic representations, either to define explanations (e.g.,} \emph{anchors} \cite{Ribeiro_18} determine sufficient conditions (inputs) for \FT{predictions (outputs))}, 
\FT{or for logic-based counterparts of the underlying models from which explanations are drawn (e.g.,   \cite{Ignatiev_19_AAAI,Ignatiev_19})}. 
\FT{\AR{Due to} their model-agnosticism, all these methods restrict explanations to ``correlations'' between \emph{inputs} and \emph{outputs}.}
\FT{Instead, our focus on a specific method (BCs) allows us to define explanations} providing a deeper representation of how 
\FT{the model} is functioning via \FT{relations between input, output and} \emph{intermediate} variables. 

Regarding 
BCs, \cite{Shih_18} define \emph{minimum cardinality} and \emph{prime implicant} explanations \FT{(extended to any model in \cite{Darwiche_20})} to ascertain pertinent 
features based on a complete set of classifications, i.e.\EAB{,} a decision function representing the BC \cite{Shih_19}. 
These explanations are 
\FT{formally defined for binary} variables \FT{only} and again only 
\FT{explain outputs in terms of inputs}. 
\FT{CFXs \cite{BC} may include also intermediate variables and their relations: we will show that they can be \AR{captured} 
as instantiations of our method.}
Explanation trees for causal Bayesian networks \cite{Nielsen_08} represent  causal relations between variables
, 
and links between causality and explanation \EAB{have} also been studied by \cite{Halpern_01_IJCAI,Halpern_01_UAI}.
\ARX{While our influences are causal wrt the BC (as opposed to some model of the real world), we do not restrict the extracted relations in IDXs exclusively to those of a causal nature.}
Finally, \cite{Timmer_15} u
se support graphs (argumentation frameworks with support relations) as explanations showing the interplay between variables (as we do) in Bayesian networks\FT{, but (differently from us) commit to specific types of relations.} 

\FT{We will explore instances of our method giving  \emph{dialectical} forms of explanations for BCs\ARX{, i.e., those which represent relationships between variables as being positive or negative}.  Various types of  argumentation-based, dialectical explanations have been defined in the literature, but 
in different contexts and of different forms than ours, e.g., }
\cite{Garcia_13} propose 
\FT{explanations as dialectical trees of arguments;}
\cite{Fan_15} explain the 
\emph{admissibility}  \cite{Dung_95} of arguments using 
 dispute trees\EAB{;} \FT{several, e.g., \cite{Teze_18,Naveed_18,rec}, draw dialectical explanations} 
to explain \FT{the outputs of}
recommender systems
, 
\FT{while others focus on argumentation-based explanations of} review aggregation \cite{RT}, 
decision-making \cite{Zeng_18} and 
scheduling \cite{Cyras_19}. 


\section{Bayesian Network Classifiers and Influences}

We \FT{first} define \AR{(discrete)} BCs \FT{and their \emph{decision functions}}
:

\begin{definition}\label{def:BC}
A \emph{BC} 
	is a tuple $\langle  \Obs, \Cla, \Val, \Dep, \Inputs \rangle$ such that:
	\begin{itemize}
	\item $\Obs$ is a (finite) set of \emph{observations};
	%
	\item $\Cla$ is a (finite) set of \emph{classifications}; we refer to $\Args=\Obs \cup \Cla$ as the set of \emph{variables};
	%
	\item $\Val$ is a set of sets such that for any $x \in \Args$ there is a unique 
	$V \in \Val$ associated to $x$, called 
	\emph{values} of $x$ ($\Val(x)$ for short); 
    \item $\Dep \subseteq \AR{\Args} \times \Args$ is 
    \FT{a} set of 
    \emph{conditional dependencies} 
    such that $\langle \Args, \Dep \rangle$ is an acyclic directed graph \FT{(we refer to this as the underlying \emph{Bayesian network})}; for any $\arga \in \Args$,  $\Dep(\arga) = \{ \argb \in \Args | \EA{(\argb, \arga)} \in \Dep \}$ are the \emph{\EA{parents}} of $\arga$;
    \item For each $\arga \in \Args$, each $\arga_i \in \Val(\arga)$ is equipped with a \emph{prior probability} $P(\arga_i) \in [0,1]$ where $\sum_{\arga_i\in \Val(\arga)} P(\arga_i) = 1$; 
    \item For each \EA{$\arga \in \Args$}, each $x_i \in \Val(x)$ is equipped with a set of \emph{conditional probabilities} 
    where \EA{if $\Dep(x) = \{ \argb, \ldots, \argc \}$, for every $\argb_m$, \ldots, $\argc_n \in \Val(\argb) \times \ldots \times \Val(\argc)$, we have $P(x_i | \argb_m$, \ldots, $\argc_n)$, again with $\sum_{x_i\in \Val(x)} P(x_i| \argb_m$, \ldots, $\argc_n) = 1$};
    \item $\Inputs$ is the set of all possible \emph{input assignments}: any $a \in \Inputs$ is a \FT{(possibly} partial\FT{)} mapping $a : \Args \mapsto \bigcup_{\arga \in \Args} \Val(x)$  such that, \FT{for every $x\in \Obs$, $a$ assigns a value $a(x) \in \Val(x)$ to $x$, and} 
    \FT{for every $x \in \Args$, for every $x_i\in \Val(x)$,} $P(x_i | a)$ \FT{is the \emph{posterior probability} of the value of $x$ being $x_i$, given $a$}.\footnote{Posterior probabilities may be  estimated 
    from the prior and conditional probabilities.
    Note that, if $a$ is defined for $x$ and $a(x)\! =\! x_i$, then $P(x_i | a)\! = \!1$ and\FT{, for all $x_j \in \Val(x) \setminus \{ x_i \}$,} $P(x_j | a)\! = \!0
    $.}
    
    %
    %
    \end{itemize}
    \FT{Then, the \emph{decision function} (of \AR{the} BC) is } $\SF: \Inputs \times \Args \mapsto \bigcup_{\arga \in \Args} \Val(\arga)$ 
    where, for any $a \in \Inputs$ and any $\arga \in \Args$, $\SF(a, \arga) = argmax_{\arga_i \in \Val(\arga)}{P(\arga_i | a )}$.
	
\end{definition}

\FT{Thus,} a BC 
consists of \emph{variables}
, which may be 
\emph{classifications} 
or \emph{observations}
, 
\emph{conditional dependencies} 
between 
\FT{them}, \emph{values} that can be ascribed to variables,
and 
associated probability distributions 
result\FT{ing} in a \emph{decision function}, i.e., a mapping from 
inputs (assignments of values to variables) to outputs (assignments of values to classifications). 
Note that, differently 
\FT{from} \cite{Shih_18,BC},  BCs are not 
\FT{equated} to decision functions\FT{, as} 
probabilistic information is 
\FT{explicit in Definition~\ref{def:BC}}. 

\FT{We will consider various concrete BCs throughout the paper, all special cases of Definition~\ref{def:BC} satisfying, in addition, an \emph{independence property} among the parents of each variable. For all these BCs, and in the remainder of the paper,  the 
}
\EA{
\textit{conditional probabilities} can be defined, for each $x \in \Args, 
x_i \in \Val(x),$ 
$y \in \Dep(x), 
y_m \in \Val(y)$, as  $P(x_i|y_m)$ with $\sum_{x_i\in \Val(x)} P(x_i| \argb_m) = 1$.
}
\FT{Specifically,}
\EA{
for single-label classification we use Naive Bayes Classifiers (NBCs), 
\FT{with} $\Cla = \{ c \}$ and $\Dep = \{ (c, x) | \arga \in \Obs \}$. For  multi-label classification we use a variant of Bayesian network-based Chain Classifiers (BCCs) \cite{ChainClassifiers}  in which leaves of the 
network are considered observations, the remaining variables classifications\FT{, and \PBX{every classification $c \in \Cla$ is} estimated} with a NBC in which the children of $c$ are the inputs.}

\FT{In the remainder of the paper, unless specified otherwise, we assume as given a generic  BC $\langle  \Obs, \Cla, \Val, \Dep, \Inputs \rangle$ satisfying the aforementioned independence property.}

\begin{figure}[t]
    \centering
    \includegraphics[width=0.48\textwidth]{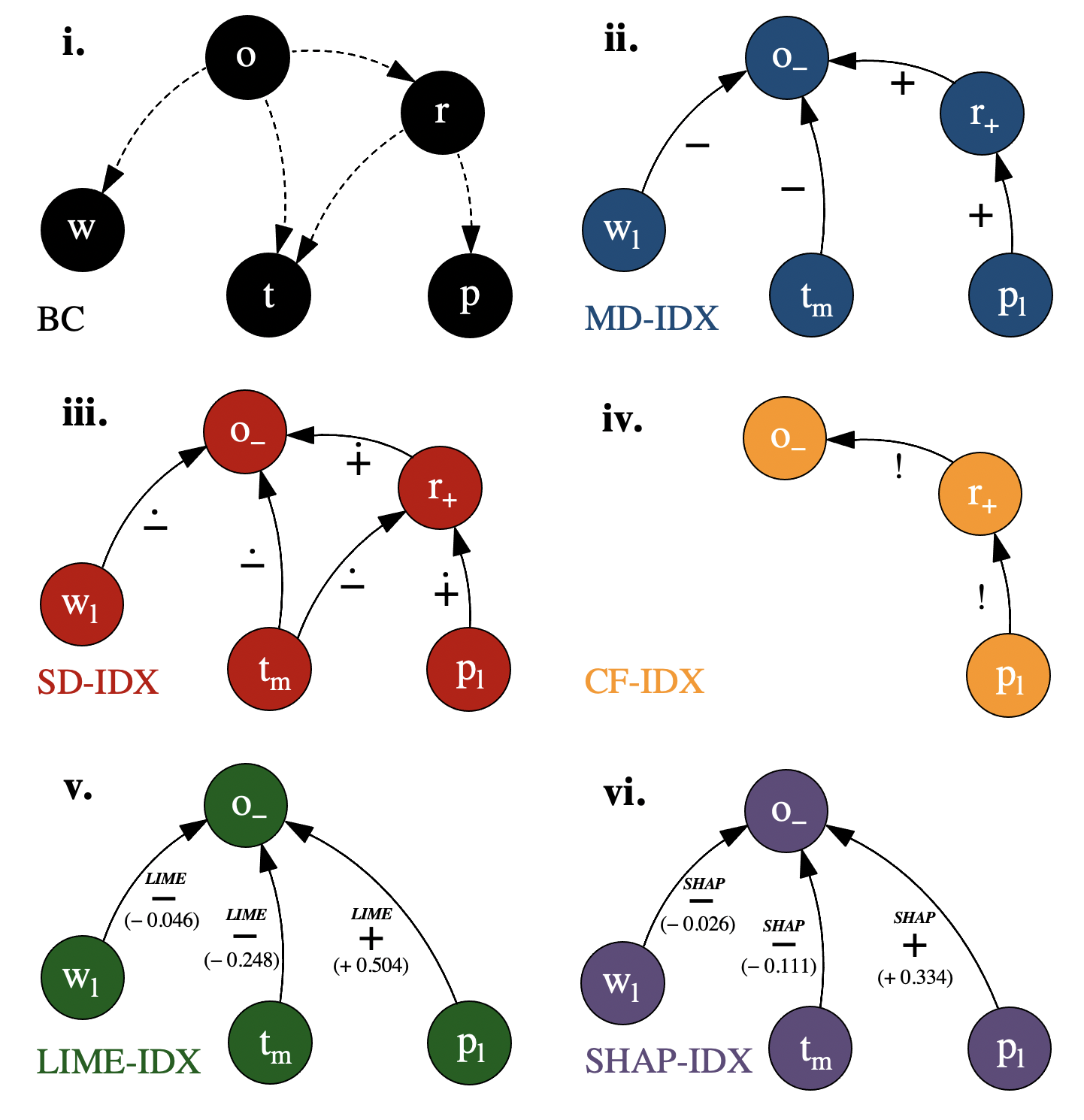}
	\caption{(i) \FT{Bayesian network for the} Play-outside BC, with conditional dependencies as dashed arrows \FT{(see Table \ref{table:probs} for the BC's probabilities and Table \ref{table:egBC} for its decision function $\SF$)}, and (ii-vi) corresponding explanations \FT{(shown as graphs, with relations given by edges labelled with their type)} for 
	input \emph{low wind} \FT{($w_l$)}, \emph{medium temperature} \FT{($t_m$)}, and \emph{low pressure} \FT{($p_l$)}, 
	with relevant variables\dell{ (with assigned values as subscripts)}, extracted relations and LIME/SHAP attribution values indicated. 
	\ARX{All explanations are instances of IDXs, illustrating our method's flexibility and potential 
	to accommodate different explanatory requirements
	.}
    \label{fig:egBC}}
\end{figure}
For illustration, consider the \emph{play-outside} \EA{BCC} in Figure \ref{fig:egBC}i \EAB{(for now ignoring the other subfigures, which will be introduced later)}, in which 
classifications \emph{play \underline{o}utside} and \emph{\underline{r}aining} are determined 
\FT{from} observations \emph{\underline{w}ind}, \emph{\underline{t}emperature} and \emph{\underline{p}ressure}.
\FT{Here,} $\Cla \!= \! \{ o, r \}$, $\Obs \!= \! \{ w, t, p \}$ and $\Dep$ \FT{is
} \EAB{as} in the figure.
Then, let $\Val$ be such that $\Val(w) = \Val(t) = \{low, medium, high\}$, $\Val(p) = \{low, high\}$ and $\Val(r) =\Val(o) = \{-, +\}$, i.e., $w$ and $t$ are categorical while $p$, $r$ and $o$ are binary.
Table \ref{table:probs} shows the prior and conditional probabilities for this \EA{B}\FT{C}C 
\FT{leading, in turn, to posterior probabilities \EAX{ in Table \ref{table:egBC}}}. For example, 
\FT{for} input 
\emph{low wind}
, \emph{medium temperature}
and \emph{low pressure}
, 
\FT{the BCC's} 
posterior probabilites for \emph{raining} and \emph{play outside} \FT{may be calculated as (with values of variables as subscripts)}\footnote{We indicate with $\propto$ the normalized 
\EAB{posterior} probability (i.e., such that $\sum_{c_i \FT{\in} \Val(c)} P(c_i | \cdot ) = 1$). 
\EAB{$P(r_+ | t_m, p_l)$ is the posterior probability} of an NBC predicting $r$ from  observations $t$, $p$ as input whilst 
\EAB{$P(o_- | w_l, t_m, p_l)$} is that of an\EAB{other} NBC predicting $o$ from, as  inputs, observations $w$, $t$ and classification $r$ (as predicted by the first NBC).}
\EA{
$P(r_+ | t_m, p_l) \EAX{= 0.94} \!\propto\! P(r_+) \cdot P(t_m | r_+) \cdot P(p_l | r_+)$ and $P(o_- | w_l, t_m, p_l) \EAX{= 0.99} \!\propto\! P(o_-) \cdot P(w_l | o_-) \cdot P(t_m | o_-) \cdot P(r_+ | o_-)$}.
This \FT{gives} the decision function $\SF$ in Table \ref{table:egBC}, where the \FT{earlier} example input results in 
\FT{$r_+$ and $o_-$}.

\begin{table}[t]
\begin{small}
\begin{center}
\begin{tabular}{ccccccccccccc}
\hline
\multicolumn{1}{c}{} & 
\!\!\!$w_l$\!\!\! &
\!\!\!$w_m$\!\!\! &
\!\!\!$w_h$\!\!\! &
\!\!\!$t_l$\!\!\! &
\!\!\!$t_m$\!\!\! &
\!\!\!$t_h$\!\!\! &
\!\!\!$p_l$\!\!\! &
\!\!\!$p_h$\!\!\! &
\!\!\!$r_+$\!\!\! &
\!\!\!$r_-$\!\!\! &
\!\!\!$o_+$\!\!\! &
\!\!\!$o_-$\!\!\! \\
\hline
\!\!\!$P(\cdot)$\!\!\! & 
\!\!\!$.33$\!\!\! &
\!\!\!$.33$\!\!\! &
\!\!\!$.33$\!\!\! &
\!\!\!$.33$\!\!\! &
\!\!\!$.33$\!\!\! &
\!\!\!$.33$\!\!\! &
\!\!\!$.50$\!\!\! &
\!\!\!$.50$\!\!\! &
\!\!\!$.67$\!\!\! &
\!\!\!$.33$\!\!\! &
\!\!\!$.22$\!\!\! &
\!\!\!$.78$\!\!\! \\
\!\!\!$P(\EA{\cdot\, | \,r_+})$\!\!\! & 
$\times$ &
$\times$ &
$\times$ &
\!\!\!$.25$\!\!\! &
\!\!\!$.25$\!\!\! &
\!\!\!$.50$\!\!\! &
\!\!\!$.75$\!\!\! &
\!\!\!$.25$\!\!\! &
$\times$ &
$\times$ &
$\times$ &
$\times$ \\
\!\!\!$P(\EA{\cdot\, | \,r_-})$\!\!\! & 
$\times$ &
$\times$ &
$\times$ &
\!\!\!$.49$\!\!\! &
\!\!\!$.49$\!\!\! &
\!\!\!$.02$\!\!\! &
\!\!\!$.02$\!\!\! &
\!\!\!$.98$\!\!\! &
$\times$ &
$\times$ &
$\times$ &
$\times$ \\
\!\!\!$P(\EA{\cdot\, | \,o_+})$\!\!\! & 
\!\!\!$.48$\!\!\! &
\!\!\!$.26$\!\!\! &
\!\!\!$.26$\!\!\! &
\!\!\!$.26$\!\!\! &
\!\!\!$.72$\!\!\! &
\!\!\!$.02$\!\!\! &
$\times$ &
$\times$ &
\!\!\!$.02$\!\!\! &
\!\!\!$.98$\!\!\! &
$\times$ &
$\times$ \\
\!\!\!$P(\EA{\cdot\, | \,o_-})$\!\!\! & 
\!\!\!$.28$\!\!\! &
\!\!\!$.36$\!\!\! &
\!\!\!$.36$\!\!\! &
\!\!\!$.36$\!\!\! &
\!\!\!$.22$\!\!\! &
\!\!\!$.42$\!\!\! &
$\times$ &
$\times$ &
\!\!\!$.85$\!\!\! &
\!\!\!$.15$\!\!\! &
$\times$ &
$\times$ \\
\hline
\end{tabular}

\textbf{\textit{}}
\end{center}
\protect\caption{Prior and conditional probabilities 
\FT{for} the play-outside \EA{BCC estimated 
with Laplace smoothing ($\alpha = 0.1$)} \FT{from the dataset/decision function  $\SF$ in Table~\ref{table:egBC}}.} \label{table:probs}
\end{small}
\end{table}

\begin{table}[t]
\begin{small}
\begin{center}
\begin{tabular}{ccccccccc}
\hline
	\multicolumn{5}{c}{\!\!$\SF$\!\!} & \multicolumn{2}{c}{SD-\EAB{/\textbf{MD-}}IDX} & 
	\multicolumn{2}{c}{CF-IDX}  \\
\!\!\!$w$\!\!\! & 
\!\!\!$t$\!\!\! & 
\!\!\!$p$\!\!\! & 
\!\!\!$r$\!\!\! & 
\!\!\!$o$\!\!\! & 
\!\!$\dot{-}(\!r\!)$/$\dot{+}(\!r\!)$\!\!\!\! &
\!\!\!$\dot{-}(\!o\!)$/$\dot{+}(\!o\!)$\!\!\! &
\!\!\!$!(\!r\!)$/$*(\!r\!)$\!\!\!\! &
\!\!\!\!$!(\!o\!)$/$*(\!o\!)$\!\!\! \\
\hline
\!\!\!$l$\!\!\! & 
\!\!\!$l$\!\!\! & 
\!\!\!$l$\!\!\! & 
\!\!\!$+${\tiny$.94$}\!\!\! & 
\!\!\!$-${\tiny$.99$}\!\!\! &
$\{ t \}$/$\{ \textit{\textbf{p}} \}$ &
$\{ \textit{\textbf{w}} \}$/$\{ t, \textit{\textbf{r}}  \}$ &
$\{ p \}$/$\{  \}$ &
$\{ r \}$/$\{ t \}$ \\
\!\!\!$m$\!\!\! & 
\!\!\!$l$\!\!\! & 
\!\!\!$l$\!\!\! & 
\!\!\!$+${\tiny$.94$}\!\!\! & 
\!\!\!$-${\tiny$.99$}\!\!\! &
$\{ t \}$/$\{ \textit{\textbf{p}} \}$ &
\!\!\!$\{  \}$/$\{ w, t, \textit{\textbf{r}} \}$\!\!\! &
$\{ p \}$/$\{  \}$ &
\!\!\!$\{  \}$/$\{ w, t, r \}$\!\!\! \\
\!\!\!$l$\!\!\! & 
\!\!\!$m$\!\!\! & 
\!\!\!$l$\!\!\! & 
\!\!\!$+${\tiny$.94$}\!\!\! & 
\!\!\!$-${\tiny$.99$}\!\!\! &
$\{ t \}$/$\{ \textit{\textbf{p}} \}$ &
$\{ \textit{\textbf{w}}, \textit{\textbf{t}} \}$/$\{ \textit{\textbf{r}} \}$ &
$\{ p \}$/$\{  \}$ &
$\{ r \}$/$\{  \}$ \\
\!\!\!$m$\!\!\! & 
\!\!\!$m$\!\!\! & 
\!\!\!$l$\!\!\! & 
\!\!\!$+${\tiny$.79$}\!\!\! & 
\!\!\!$-${\tiny$.51$}\!\!\! &
$\{ t \}$/$\{ \textit{\textbf{p}} \}$ &
$\{ \textit{\textbf{t}} \}$/$\{  w, \textit{\textbf{r}} \}$ &
$\{ p \}$/$\{  \}$ &
$\{ r \}$/$\{ w \}$ \\
\!\!\!$h$\!\!\! & 
\!\!\!$l$\!\!\! & 
\!\!\!$l$\!\!\! & 
\!\!\!$+${\tiny$.79$}\!\!\! & 
\!\!\!$-${\tiny$.69$}\!\!\! &
$\{ t \}$/$\{ \textit{\textbf{p}} \}$ &
\!\!\!$\{  \}$/$\{ w, t, \textit{\textbf{r}} \}$\!\!\! &
$\{ p \}$/$\{  \}$ &
\!\!\!$\{  \}$/$\{ w, t, r \}$\!\!\! \\
\!\!\!$l$\!\!\! & 
\!\!\!$h$\!\!\! & 
\!\!\!$l$\!\!\! & 
\!\!\!$+${\tiny$.79$}\!\!\! & 
\!\!\!$-${\tiny$.51$}\!\!\! &
$\{  \}$/$\{ \textit{\textbf{t}}, \textit{\textbf{p}} \}$ &
$\{ \textit{\textbf{w}} \}$/$\{ \textit{\textbf{t}}, \textit{\textbf{r}} \}$ &
$\{  \}$/$\{ t, p \}$ &
$\{  \}$/$\{ r \}$ \\
\!\!\!$h$\!\!\! & 
\!\!\!$m$\!\!\! & 
\!\!\!$l$\!\!\! & 
\!\!\!$+${\tiny$.79$}\!\!\! & 
\!\!\!$-${\tiny$.81$}\!\!\! &
$\{ t \}$/$\{ \textit{\textbf{p}} \}$ &
$\{ \textit{\textbf{t}} \}$/$\{ w, \textit{\textbf{r}} \}$ &
$\{ p \}$/$\{  \}$ &
$\{ r \}$/$\{ w \}$ \\
\!\!\!$m$\!\!\! & 
\!\!\!$h$\!\!\! & 
\!\!\!$l$\!\!\! & 
\!\!\!$+${\tiny$.79$}\!\!\! & 
\!\!\!$-${\tiny$.91$}\!\!\! &
$\{  \}$/$\{ \textit{\textbf{t}}, \textit{\textbf{p}} \}$ &
\!\!\!$\{  \}$/$\{ w, \textit{\textbf{t}}, \textit{\textbf{r}} \}$\!\!\! &
$\{  \}$/$\{ t, p \}$ &
$\{  \}$/$\{ w, r \}$ \\
\!\!\!$h$\!\!\! & 
\!\!\!$h$\!\!\! & 
\!\!\!$l$\!\!\! & 
\!\!\!$+${\tiny$.79$}\!\!\! & 
\!\!\!$-${\tiny$.81$}\!\!\! &
$\{  \}$/$\{ \textit{\textbf{t}}, \textit{\textbf{p}} \}$ &
\!\!\!$\{  \}$/$\{ w, \textit{\textbf{t}}, \textit{\textbf{r}} \}$\!\!\! &
$\{  \}$/$\{ t, p \}$ &
$\{  \}$/$\{ w, r \}$ \\
\!\!\!$l$\!\!\! & 
\!\!\!$l$\!\!\! & 
\!\!\!$h$\!\!\! & 
\!\!\!$-${\tiny$.99$}\!\!\! & 
\!\!\!$+${\tiny$.99$}\!\!\! &
$\{  \}$/$\{ t, \textit{\textbf{p}} \}$ &
\!\!\!$\{  \}$/$\{ \textit{\textbf{w}}, t, \textit{\textbf{r}} \}$\!\!\! &
$\{ p \}$/$\{ t \}$ &
$\{ w, r \}$/$\{  \}$ \\
\!\!\!$m$\!\!\! & 
\!\!\!$l$\!\!\! & 
\!\!\!$h$\!\!\! & 
\!\!\!$-${\tiny$.99$}\!\!\! & 
\!\!\!$-${\tiny$.99$}\!\!\! &
$\{  \}$/$\{ t, \textit{\textbf{p}} \}$ &
$\{ t, \textit{\textbf{r}} \}$/$\{ w \}$ &
$\{ p \}$/$\{ t \}$ &
$\{ t \}$/$\{ w \}$ \\
\!\!\!$l$\!\!\! & 
\!\!\!$m$\!\!\! & 
\!\!\!$h$\!\!\! & 
\!\!\!$-${\tiny$.99$}\!\!\! & 
\!\!\!$+${\tiny$.99$}\!\!\! &
$\{  \}$/$\{ t, \textit{\textbf{p}}  \}$ &
\!\!\!\!$\{  \}$/$\{ \textit{\textbf{w}},\textit{\textbf{t}}, \textit{\textbf{r}} \}$\!\!\!\! &
$\{ p \}$/$\{ t \}$ &
\!\!\!$\{ r \}$/$\{ w, t \}$\!\!\! \\
\!\!\!$m$\!\!\! & 
\!\!\!$m$\!\!\! & 
\!\!\!$h$\!\!\! & 
\!\!\!$-${\tiny$.97$}\!\!\! & 
\!\!\!$+${\tiny$.99$}\!\!\! &
$\{ \}$/$\{ t, \textit{\textbf{p}} \}$ &
\!\!\!$\{ w \}$/$\{  \textit{\textbf{t}}, \textit{\textbf{r}} \}$\!\!\! &
$\{ p \}$/$\{ t \}$ &
$\{ t, r \}$/$\{  \}$ \\
\!\!\!$h$\!\!\! & 
\!\!\!$l$\!\!\! & 
\!\!\!$h$\!\!\! & 
\!\!\!$-${\tiny$.97$}\!\!\! & 
\!\!\!$-${\tiny$.99$}\!\!\! &
$\{  \}$/$\{ t, \textit{\textbf{p}} \}$ &
\!\!\!$\{ t, \textit{\textbf{r}} \}$/$\{ w \}$\!\!\! &
$\{ p \}$/$\{ t \}$ &
$\{ t \}$/$\{ w \}$ \\
\!\!\!$l$\!\!\! & 
\!\!\!$h$\!\!\! & 
\!\!\!$h$\!\!\! & 
\!\!\!$+${\tiny$.97$}\!\!\! & 
\!\!\!$-${\tiny$.99$}\!\!\! &
$\{ \textit{\textbf{p}} \}$/$\{ \textit{\textbf{t}} \}$ &
$\{ \textit{\textbf{w}} \}$/$\{ \textit{\textbf{t}}, \textit{\textbf{r}} \}$ &
$\{ t \}$/$\{  \}$ &
$\{  \}$/$\{ r \}$ \\
\!\!\!$h$\!\!\! & 
\!\!\!$m$\!\!\! & 
\!\!\!$h$\!\!\! & 
\!\!\!$-${\tiny$.97$}\!\!\! & 
\!\!\!$+${\tiny$.98$}\!\!\! &
$\{  \}$/$\{ t, \textit{\textbf{p}} \}$ &
$\{ w \}$/$\{ \textit{\textbf{t}}, \textit{\textbf{r}} \}$ &
$\{ p \}$/$\{ t \}$ &
$\{ t, r \}$/$\{  \}$ \\
\!\!\!$m$\!\!\! & 
\!\!\!$h$\!\!\! & 
\!\!\!$h$\!\!\! & 
\!\!\!$+${\tiny$.97$}\!\!\! & 
\!\!\!$-${\tiny$.95$}\!\!\! &
$\{ \textit{\textbf{p}} \}$/$\{ \textit{\textbf{t}} \}$ &
\!\!\!$\{  \}$/$\{ w, \textit{\textbf{t}}, \textit{\textbf{r}} \}$\!\!\! &
$\{ t \}$/$\{  \}$ &
$\{  \}$/$\{ w, r \}$ \\
\!\!\!$h$\!\!\! & 
\!\!\!$h$\!\!\! & 
\!\!\!$h$\!\!\! & 
\!\!\!$+${\tiny$.97$}\!\!\! & 
\!\!\!$-${\tiny$.98$}\!\!\! &
$\{ \textit{\textbf{p}} \}$/$\{ \textit{\textbf{t}} \}$ &
\!\!\!$\{  \}$/$\{ w, \textit{\textbf{t}},\textit{\textbf{r}} \}$\!\!\! &
$\{ t \}$/$\{  \}$ &
$\{  \}$/$\{ w, r \}$ \\
\hline
\end{tabular}

\textbf{\textit{}}
\end{center}
\protect\caption{\FT{Decision function $\SF$} \EAB{(with the probability of the most probable class)} and explanations \FT{(SD-IDX and CF-IDX, \EAB{that will be introduced in Section \ref{sec:IDX}})} for 
the play-outside BC \EAB{in Figure \ref{fig:egBC}i}, where, 
for any variable $x \in \Args$ and 
relation \AR{type $t$}, $\AR{t}(x) = \{ y \in \Args | (y, x) \in \Rels_{\AR{t}}\}$.
We indicate 
\FT{in bold variables} 
\FT{{in the corresponding monotonic attackers and supporters \EAB{sets} in the MD-IDX for this example}}
\EAB{(also introduced in Section \ref{sec:IDX})}. 
}\label{table:egBC}
\end{small}
\end{table}

Our method for generating explanations relies on modelling how the variables within a BC \emph{influence} one another. For this we adapt the following from \cite{BC}.

\begin{definition}\label{def:Influences}
The set of 
\emph{influences} 
is the \FT{(}acyclic\FT{)} relation 
	$\Influences = \{(\arga, c) \in \Args \times \Cla | (c,\arga) \in \Dep \}$. 
	%
\end{definition}

Influences 
thus indicate the direction of the inferences 
in determining the values of 
classifications, neglecting, for example, dependencies between observations \FT{as considered} in \emph{tree-augmented naive BCs} \cite{Friedman_97}. 
Note that other forms of influence\EAB{, which we leave to future work, }may be required 
for more complex BCs 
\FT{where} inferences are made in other ways, e.g., \FT{for} \emph{Markov Blanket-based BCs} \cite{Koller_96} 
\FT{where} inferences may 
\FT{follow} 
dependencies or \FT{be drawn} 
between variables not linked by 
\FT{them}. 
\EAB{Notice also that \PBX{the influences correspond to the structure of the model as it has been built or learned from data}, any imprecision in the construction of the model or bias in the data with respect to the real world will be reflected in the influences too}.

While 
\FT{the} 
explanations \FT{we will define} may 
\FT{amount to} non-shallow graph\FT{s}\EAB{\footnote{\EAB{A non-shallow graph is a graph 
\FTx{with one or more paths (}between \FTx{any} two vertices\FTx{)} of length $>1$.}}} of relations between 
variables, in order to 
\FT{integrate} other, shallow \FT{explanation} methods, \ARX{i.e.,} 
\FT{connecting only} inputs and outputs, we define a restricted form of 
\FT{influences}.

\begin{definition}\label{def:IOInfluences}
The set of 
\emph{input-output influences}
, where $\Cla_o \subseteq \Cla$ are \emph{outputs}, is the \FT{(}acyclic\FT{)} relation $\Influences_{io} = \Obs \times \Cla_o$.
\end{definition}

For illustration, in the running example in Figure \ref{fig:egBC}i, $\Influences = \{(w,o),(t,o),(r,o),(t,r),(p,r) \}$ and $\Influences_{io} = \{(w,o),(t,o),(p,o) \}$ 
\FT{for} $\Cla_o = \{ o \}$\AR{, while $\Influences_{io} = \{(w,r),$ $(t,r),(p,r),(w,o),(t,o),(p,o) \}$ 
\FT{for} $\Cla_o = \{ o, r \}$. Note that in the former $\Influences_{io}$ case, $r$ is neglected, while in the latter, the influence $(w,r)$ is extracted despite the fact that \textit{wind} cannot influence \textit{raining} in this BC, highlighting}
\FT{that using $\Influences_{io}$, instead of the full $\Influences$, \ARR{may have drawbacks} for non-naive BCs, except when the notions coincide, as characterised next:}\footnote{\FT{The proofs of all results are in 
\EA{Appendix A}.}}

\begin{proposition}\label{thm:influences}
    \FT{Given} outputs $\Cla_o \subseteq \Cla$, $\Influences = \Influences_{io}$
    iff
    $\Dep = \Cla_o \times \Obs$.
\end{proposition}


\section{\EAB{Influence}-Driven Explanations}\label{sec:IDX}





\FT{We} introduce a general method for generating explanations of BCs, 
\FT{and then several instantiations thereof}.
Our explanations are constructed by categorising 
influences as specific \FT{types} 
of relations depending on 
\FT{the satisfaction of \ARX{relation} properties characterising those relations}, defined as follows.

\begin{definition}
\FT{Given 
influences $\Influences$,}
an \emph{explanation kit} is a finite set of pairs $\set{\tuple{t_1,\prop_1}, \ldots  \tuple{t_n,\prop_n}}$ 
where $\set{t_1, \ldots, t_n}$ is a set of \emph{relation types} and $\set{\prop_1, \ldots, \prop_n}$ is a set of \emph{relation properties}
\FT{with} $\prop_i: \Influences \times \Inputs \rightarrow \{ true, false \}$, 
\FT{for $i =1, \ldots, n$}.
\end{definition}
\FT{Intuitively, relation property $\prop_i$ is satisfied for $(x,y)\in \Influences$ and $a\in \Inputs$ iff $\prop((x,y),a) = true$.}

\begin{definition}
Given 
influences $\Influences$ and an explanation kit $\{ \tuple{t_1,\prop_1}, \ldots  \tuple{t_n,\prop_n}\}$, a\ARX{n} \emph{\ARX{influence}-driven explanation} (IDX) for 
\emph{explanandum} $\expvar \in \Cla$ with 
input assignment $a \in \Inputs$ is a tuple $\langle \RV,\Rels_{\AR{t_1}}, \ldots, \Rels_{\AR{t_n}} \rangle$ 
\FT{with}:\\
    $\bullet$ $\RV \subseteq \Args$ such that $\expvar \in \RV$; \\
    %
	$\bullet$ $\Rels_
	{\FT{t_1}}, \ldots  \Rels_
	{\FT{t_n}}\subseteq \Influences \cap (\RV \times \RV)$ such that for any $i=1 \ldots n$, for every $(\arga, \argb) \in \Rels_{\FT{t_i}}$, $\prop_i((\arga,\argb), a)=true$
	;\\
	%
	$\bullet$ $\forall \arga \in \RV$ there is a sequence $\arga_1, \ldots, \arga_k$, $k \geq 1$, such that $\arga_1=\arga$, $\arga_k=\expvar$, and $\forall 1 \leq i < k$ $(\arga_i, \arga_{i+1}) \in \Rels_{\FT{t_1}}\cup\ldots\cup\Rels_{\FT{t_n}}$.
\end{definition}

A\ARX{n} IDX is thus guaranteed to be a set of \emph{relevant} variables \FT{($\RV$)}, including the explanandum, connected to one another in a graph by the relations 
\FT{from} the explanation kit. \AR{We follow the findings of \cite{Nielsen_08} that only some of the variables in $\Args$ (potentially very few \cite{Miller_19}) may suffice to explain the 
\FT{explanandum's value}.}

\begin{algorithm}[ht]
    \begin{algorithmic}
        \Function{Parents}{$\expvar, \Influences$}
            \State $parents \gets \emptyset$
            \For{$(x, c) \in \Influences$}
                \If{$c == e$}
                    \State $parents \gets parents \cup \{ x \}$
                \EndIf
            \EndFor
            \State \Return{$parents$}
        \EndFunction\\
        \Function{XP}{$\Obs, \Cla, \Influences, \expvar, a, \{ \tuple{t_1,\prop_1}, \ldots , \tuple{t_n,\prop_n} \}$}
            \Comment{{Generate the explanation for the explanandum $\expvar \in \Cla$ given an {input assignment $a \in \Inputs$} using the explanation kit $\{ \tuple{t_1,\prop_1}, \ldots , \tuple{t_n,\prop_n} \}$ }}
            \State $\RV \gets \emptyset$
            \State $\Rels \gets \{\}$
            \For{$t_i \in \{ t_1, \ldots, t_n \}$}
                \State $\Rels[t_i] = \emptyset$
            \EndFor
            \For{$x \in \Call{Parents}{\expvar, \Influences}$}
                \For{$\prop_i \in \{ \prop_1, \ldots, \prop_n \} $}
                    \If{$\prop_i((x,\expvar),a)$}
                        \State $\RV \gets \RV \cup \{ x \}$
                        \State $\Rels[t_i] \gets \Rels[t_i] \cup \{ (x,\expvar) \}$
                    \EndIf
                \EndFor
               \State {$\widehat{\RV}, \widehat{\Rels} \gets$ \Call{XP}{$\Obs, \Cla, \Influences, x, a, \{ \tuple{t_1,\prop_1}, \ldots , \tuple{t_n,\prop_n} \}$}}
                \State{$\RV \gets  \RV \cup \widehat{\RV}$}
                \For{$t_i \in \{ t_1, \ldots, t_n \}$}
                    \State{$\Rels[t_i] \gets \Rels[t_i] \cup \widehat{\Rels[t_i]}$}
                \EndFor
            \EndFor
            \State \Return{$\RV, \Rels$}
        \EndFunction
    \end{algorithmic}
    \caption{Influence-Driven Explanations}\label{algo:algorithm}
\end{algorithm}

\EAB{The computational cost of \PBX{producing} IDXs \PBX{essentially depends on} the computational cost of computing the relation properties \EAX{(for each influence), namely, once an input assignment $a$ is given, $\prop_1((x,y),a), \ldots, \prop_n((x,y))$ for every influence $(x,y) \in \Influences$}. In fact, we can see from Algorithm~\ref{algo:algorithm} that the \EAX{rest of the }explanation process mainly amount\FTx{s} to a deep-first search in the graph of the influences. \EAX{We will discuss further the computational costs of our explanations in Section \ref{sec:empirical}.}}

\AR{We will now demonstrate the flexibility of 
\ARX{our approach} by instantiating}
various IDXs, which are 
\FT{illustrated} in Table \ref{table:egBC} and in Figures \ref{fig:egBC}ii-vi for the running example. In doing so, we will make use of the following \FT{notion}.

\begin{definition}\label{def:ModifiedInput}
Given a BC $\langle  \Obs, \Cla, \Val, \Dep, \Inputs \rangle$ with influences $\Influences$, \FT{a variable $x \in \Args$} and an input $a \in \Inputs$, \FT{the} \emph{modified input} $a'_{x_k} \in \Inputs$ 
\FT{by} $x_k \in$ 
\FT{$\Val(x)$} \FT{is such that, for any $z\in \Args$}:  
$a'_{x_k}(z) =x_k$ if $z = x$, and $a'_{x_k}(z) =a(z)$ otherwise.
\end{definition}
A modified input thus assigns a desired value \FT{($x_k$)} to a specified variable \FT{($x$)}
\FT{, keeping} the preexisting input assignments \FT{unchanged}. \AR{For example, if \FT{
$a \in \Inputs$ amounts to} \emph{low wind}, \emph{medium temperature} and \emph{low pressure} 
in the running example 
\FT{(i.e., $a(w) = l$, $a(t) = m$, $a(p) = l$)}, then $a_{w_h}' \in \Inputs$ refers to 
\emph{high wind}, \emph{medium temperature} and \emph{low pressure}.}

\subsection{\FT{Monotonically} Dialectical IDXs}\label{sec:MIDX}
Motivated by 
\FT{recent uses} of argumentation 
for explanation (
see Section \ref{sec:related}), 
\FT{we draw inspiration from 
\emph{bipolar argumentation frameworks} \cite{Cayrol:05} to define a \emph{dialectical} instance of the notion of explanation kit, with attack and support relations defined by imposing properties}
of \emph{dialectical monotonicity}
, 
\FT{requiring} that attackers (supporters) have a \EAB{negative (positive}, \resp) effect on the variables they influence. 
Concretely, we require that an influencer is an attacker (a supporter) if its assigned value 
minimises (maximises, \resp) the posterior probability of the influencee's current value (with all other influencers' values  unchanged):

\begin{definition}\label{def:DIDX}
A \emph{\EA{monotonically} dialectical explanation kit} is a pair $\set{\tuple{-,\prop_{-}}, \tuple{+,\prop_{+}}}$ 
\FT{with} relation \FT{types} of \emph{\AR{monotonic} attack} $-$ and \emph{\AR{monotonic} support} $+$ 
characterised by the following \FT{relation} properties $\prop_-$, $\prop_+$. For any $(\arga,\argb) \in \Influences$,  $a \in \Inputs$
:
 
    \noindent $\bullet$ $\prop_{-}((\arga,\argb),a)\!= true$ iff $\forall x_k \in \Val(x) \setminus \{ \SF(a, x) \}$:
    
        \hspace*{1cm} \(P(\SF(a, y) | a)
        < 
        P(\SF(a, y) | a'_{x_k}) \);
    
    \noindent $\bullet$ $\prop_{+}((\arga,\argb),a)\!= true$ iff $\forall x_k \in \Val(x) \setminus \{ \SF(a, x) \}$:
    
    \hspace*{1cm} \(
        P(\SF(a, y) | a)
        >
        P(\SF(a, y) | a'_{x_k})
        \).
 

\end{definition}

Thus, a \emph{\EA{monotonically} dialectical IDX (MD-IDX)} is a IDX 
\FT{drawn from}
a \AR{monotonically} dialectical 
kit.

\FT{For illustration, consider} 
the MD-IDX in Figure \ref{fig:egBC}ii\FT{: here } 
\FT{$p_l$ monotonically} supports (i.e., increases the probability of) 
$r_+$, which in turn \AR{monotonically} supports 
$o_-$; while,  
$w_l$ and 
$t_m$ reduce the chance of 
$o_-$, 
\FT{leading to a monotonic attack}.

\ARX{It should be noted that since the dialectical monotonicity requirement here is rather strong, for some BCs (in particular those with variables with large domains) the MD-IDX could be empty, i.e., \PBX{comprises only the explanandum}. 
We examine the prevalence of these relations for a range of datasets in Section \ref{sec:evaluation}.}
\PBX{This form of explanation is appropriate in contexts where the users prefer sharp explanations with monotonic properties, when available, and may accept the absence of explanations, when these properties are not satisfied.}

\subsection{Stochastically Dialectical IDXs}

\FT{We believe that}
\EA{
monotonicity is a \EAB{property} 
humans naturally expect from 
explanations 
}.
\FT{Nonetheless, monotonicity is a strong requirement that may lead, for some BCs and contexts, to very few influences playing a role in MD-IDXs.}
\AR{We now introduce a \FT{weaker form of dialectical} explanation
, where}
an influencer is an attacker (supporter) if the posterior probability of the influencee's current value is lower (higher, \resp) than the average of those resulting {from the influencer's other values}, weighted by their prior probabilities (while all other influencers' values remain unchanged):

\begin{definition}\label{def:SDIDX}
A \emph{stochastically dialectical explanation kit} is a pair
$\set{\tuple{\minusSTAT,\prop_{\minusSTAT}}, \tuple{\plusSTAT,\prop_{\plusSTAT}}}$ 
\FT{with} relation \FT{types} of \emph{stochastic attack} $\minusSTAT$ and \emph{stochastic support} $\plusSTAT$ 
characterised by the following \FT{relation} properties \FT{$\prop_{\minusSTAT}$, $\prop_{\plusSTAT}$}. For any $(\arga,\argb) \in \Influences$, $a \in \Inputs$:

    \noindent $\bullet$ $\prop_{\minusSTAT}((\arga,\argb),a)\!= true$ iff
    
    \hfill \(
        \!\!\!\!\!\!P(\SF(a, y) | a)
        \!<\!
        \frac{
        \sum\limits_{x_k \in \Val(x) \setminus \{ \SF(a, x) \}} \left[ P(x_k) \cdot
        P(\SF(a, y) | a'_{x_k}) \right]
        }{
        \sum\limits_{x_k \in \Val(x) \setminus \{ \SF(a, x) \}}  P(x_k) 
        }
    \);
    
    \noindent $\bullet$ $\prop_{\plusSTAT}((\arga,\argb),a)\!= true$ iff
    
    \hfill \(
        \!\!\!\!\!\!P(\SF(a, y) | a)
        \!>\! 
        \frac{
        \sum\limits_{x_k \in \Val(x) \setminus \{ \SF(a, x) \}} \left[ P(x_k) \cdot
        P(\SF(a, y) | a'_{x_k})
        \right]
        }{
        \sum\limits_{x_k \in \Val(x) \setminus \{ \SF(a, x) \}}  P(x_k) 
        }
\).
\end{definition}

\EAB{Then,} a \emph{stochastically dialectical IDX (SD-IDX)} is a IDX 
\FT{drawn from} a stochastically dialectical 
kit. \EA{SD-IDXs are \emph{stochastic} in that the\FT{y weaken the} monotonicity constraint 
\FT{(compared to} MD-IDX\FT{s}) 
by taking into 
\FT{account} the prior probabilities of the possible changes of the influencers. }

\FT{For illustration, t}he SD-IDX in Figure \ref{fig:egBC}iii 
\FT{extends the} MD-IDX \FT{in Figure \ref{fig:egBC}ii by} 
\FT{including} the negative (stochastic) effect which 
$t_m$ has on 
$r_+$.

\ARX{The weakening of the \FTx{monotonicity} requirement 
\FTx{for SD-IDXs} mean\FTx{s} that SD-IDXs will not be empty except in some special, improbable cases, e.g., when every influencing variable of the explanandum has all values with equal posterior probability.}
\PBX{We therefore expect this form of explanation to be appropriate in contexts where users are looking for explanations featuring some degree of dialectically monotonic behaviour in the influences, but do not require strong properties and then would prefer to receive a more populated IDX than MD-IDXs.}

\subsection{Counterfactual IDXs}

We \FT{can} also 
\FT{naturally instantiate explanation kits to define counterfactual explanations capturing the} CFXs of \cite{BC}
. \FT{For this} we 
\FT{use} two 
relation\FT{s}: one indicating influencers wh
\FT{ose} assigned value is critical to the influencee's current value, and another indicating
\PBX{influencers whose assigned value can potentially contribute to the change (together with the changes of other influencers)}. 

\begin{definition}\label{def:CFIDX}
A \emph{counterfactual explanation kit} is a pair $\set{\tuple{!,\prop_!}, \tuple{*,\prop_*}}$ 
\FT{with} relation \FT{types} of \emph{critical influence} $!$ and \emph{potential influence} $*$ 
characterised by the following \FT{relation} properties $\prop_!$, $\prop_*$. For any $(\arga,\argb) \in \Influences$,  $a \in \Inputs$:

    \noindent $\bullet$ $\prop_!(\!(\!\arga,\!\argb\!),a\!)\!\!=\!\!true$ iff 
        $\exists a' \!\!\in\!\! \Inputs$ such that $\SF(a'\!,\! \arga) \!\!\neq\!\! \SF(a, \arga)$ and $\forall \argc \!\in\! \Influences(\argb) \backslash \{\! \arga\! \}$ $\SF(\!a'\!,\! \argc\!) \!\!=\!\! \SF(\!a,\! \argc\!)$, and $\forall a' \!\!\in\! \Inputs$,
    $\SF(\!a,\! \argb\!) \!\!\neq\!\! \SF(\!a'\!,\! \argb\!)$.
    
    \noindent $\bullet$ $\prop_*\!(\!(\!\arga,\!\argb\!),a\!)\!=\!true$ iff $\prop_!\!(\!(\!\arga,\!\argb\!),a\!)\!=\!false$ and $\exists a', a'' \in \Inputs$ such that $\SF(a, \arga) = \SF(a', \arga) \neq \SF(a'', \arga)$, $\forall \argc \in \Influences(\argb) \backslash \{ \arga \}$ $\SF(a', \argc) = \SF(a'', \argc)$, and 
    $\SF(a, \argb) = \SF(a', \argb) \neq \SF(a'', \argb)$.

\end{definition} 

Thus a \emph{counterfactual IDX (CF-IDX)} is a\ARX{n} IDX 
\FT{drawn from} a counterfactual 
kit. 
\FT{For illustration, }
Figure \ref{fig:egBC}iv shows the CF-IDX for the running example, 
\FT{with} no potential influences 
\FT{and with critical influences indicating} that if 
\ARR{$p_l$ were to change (i.e., to $p_h$), this would force $r_-$ and, in turn, $o_+$.}
\dell{the \emph{pressure} were to change, this would force \emph{raining} to be negative and, in turn, \emph{play outside} to be positive
.}

\ARX{CF-IDXs are not built around the principle of dialectical monotonicity, and instead reflect the effects that changes to the values of variables would have on the variables that they influence.}
\PBX{We suggest that this form of explanation is appropriate when the users wish to highlight the factors which led to a prediction and \FTx{which} can be changed in order to reverse it.}

\subsection{\EA{Attribution Method Based} Dialectical IDXs}\label{sec:otherdialecticalIDX}

We now \FT{further show the versatility of the notion of explanation kit by} instantiat\FT{ing it to integrate} 
attribution method\ARX{s, e.g., LIME} 
\FT{(see Section~\ref{sec:related})}, 
\ARX{which are} widely used 
\FT{in academia} and industry.
To \FT{reflect \ARX{attribution methods'} 
focus on input-output variables, 
\ARX{these instances are} defined in terms of input-output influences $\Influences_{io}$ and outputs $\Cla_o$ as follows:} 

\begin{definition}
A \emph{LIME explanation kit} is a pair
$\set{\tuple{\minusLIME~\!\!\!\!\!,~\prop_{\minusLIME}},~\tuple{\plusLIME~\!\!\!\!\!,~\prop_{\plusLIME}}}$ 
with relation types 
\emph{LIME-attack} $\minusLIME$ and \emph{LIME-support} $\plusLIME$ characterised by the following relation properties 
$\prop_{\minusLIME}$, $\prop_{\plusLIME}$. For any $(\arga,\argb) \in \Influences_{io}$, $a \in \Inputs$:
    
     \noindent $\bullet$ $\prop_{\plusLIME}((x,y), a) = true 
            \text{ iff }
            \limefun(a, x, y) > 0$
     , and
     
      \noindent $\bullet$ $\prop_{\minusLIME}((x,y), a)  = true 
            \text{ iff }
            \limefun(a, x, y) < 0$,
    \\
    where $\limefun : \Obs \times \Inputs \times \Cla_o \mapsto \mathbb{R}$ is such that 
    $\limefun(a, x, y)$ is the value that LIME assigns to the observation $x$ with an input assignment $a$ with respect to the output value $\SF(a, y)$.
\end{definition} 

\EAB{Then}, a \emph{LIME-IDX} is a\ARX{n} IDX 
drawn from a LIME explanation kit. 
\FT{Note that LIME-IDXs are still \emph{dialectical}. We can similarly define
(dialectical) explanation kits and IDXs for other attribution methods (see Section~\ref{sec:related}), e.g.,}
 a \textit{SHAP explanation kit}  and resulting 
\emph{SHAP-IDX}. 
%
For illustration, Figures \ref{fig:egBC}v and \ref{fig:egBC}vi show 
\FT{the LIME-IDX and SHAP-IDX}, \resp, for the play-outside BC\FT{C, with input-output influences annotated with the respective attribution values}. %
\PBX{A notable difference is that LIME and SHAP also provide weightings on the relations (LIME contributions and Shapley values, resp.), while IDXs do not include any such scoring.
Enhancing IDXs with quantitative information is an important direction of future work. As an example, in SD-IDXs importance scores could be defined based on the difference between $P(\SF(a,y)|a)$ and its (weighted) expected probability with other inputs (the right-hand side of the inequalities in Definition \ref{def:SDIDX}).}

\FT{The restriction to input-output influences implies that} the intermediate variable \emph{raining} is not considered by 
these explanations. This may not seem essential in a simple example such as this, but in real world applications such as medical diagnosis, where BCs are particularly prevalent,  
\AR{
the inclusion of 
\FT{intermediate} information could be beneficial (e.g., see 
\FT{Figure~\ref{fig:IDXeg})}}. 
\PBX{We suggest that these forms of IDX are suitable when the users prefer explanations with a simpler structure and, in particular, are not concerned about the intermediate variables of the BC nor require the dialectical relations to satisfy dialectical monotonicity.}

\section{Evaluation}\label{sec:evaluation}

In this section we perform a thorough evaluation of IDXs via theoretical \ARX{and} empirical analyses.

\subsection{Theoretical Analysis}


Our first two propositions show the relationship and, in special cases, equivalence \PBX{between} MD-IDXs and SD-IDXs.

\begin{proposition}\label{thm:DIDXtoSDIDX}
    Given an MD-IDX $\langle\! \Args_r,\! \Atts, \! \Supps \!\rangle$ and an SD-IDX $\langle\! \Args'_r,\! \Rels_{\minusSTAT}\!,\! \Rels_{\plusSTAT} \!\rangle$ for an explanandum $e\!\! \in\!\! \Args_r \!\cap\! \Args'_r$ 
    \FT{and} input assignment $a \!\in\! \Inputs$, we have 
        $\Args_r \!\!\subseteq\!\! \Args'_r
        \text{, }
        \Atts \!\!\subseteq\!\! \Rels_{\minusSTAT} 
        \text{ and } 
        \Supps \!\!\subseteq\!\! \Rels_{\plusSTAT}$.
\end{proposition}

\EAB{Proposition~\ref{thm:DIDXtoSDIDX} shows that an MD-IDX (for a certain explanandum and input assignment) is always 
(element-wise) \FTx{a} subset of the corresponding SD-IDX.}

\EAB{As described in Section \ref{sec:MIDX}, \ARX{due to} 
the stronger assumption \ARX{required by} \emph{dialectical monotonicity} 
(when compared to its stochastic counterpart), there are 
two \ARX{predominant} factors that contribute to MD-IDX becoming an (increasingly) smaller\ARX{, and potentially empty,} subset of SD-IDX
: (1) the cardinalities of variables' domains\ARX{, i.e.,} the larger the domain the less likely \emph{dialectical monotonicity} is to hold; (2) the depth of the explanation, i.e., the deeper an explanation (i.e, the longer the path from the inputs to the explanandum is), the less likely it is for a variable to have a path to the explanandum.}
\EAB{\PBX{When all variables are binary}, MD-IDXs and SD-IDXs are equivalent as shown in the following.}

\begin{proposition}\label{thm:equivalence}
    Given an MD-IDX $\langle \Args_r, \Atts, \Supps \rangle$ and an SD-IDX $\langle \Args'_r,\! \Rels_{\minusSTAT},\! \Rels_{\plusSTAT} \rangle$ for 
    explanandum $e \in \Args_r \cap \Args'_r$ 
    \FT{and} 
    input assignment $a \in \Inputs$, if, for any $x \in \Args'_r \setminus \{ e \}$, $|\Val(x)| = 2$, then
        $\Args_r = \Args'_r
        \text{, }
        \Atts = \SDAtts 
        \text{ and }
        \Supps = \SDSupps$.
\end{proposition}

\ARX{We now characterise 
desirable behaviour of
\FT{dialectical} explanation kits, requiring that attackers and supporters have a monotonic effect on the posterior probability of the assigned values to variables that they influence.
Concretely, we characterise a \emph{dialectical} version of \emph{monotonicity} which requires that if the influence from a variable is classified as being a support (an attack) then its assigned value maximises (minimises, \resp) the posterior probability of the influenced variable's current value (with all other influencing variables' values remaining unchanged).}
\EAB{In other words, if we were to change the value of a variable supporting (attacking) another variable, then the posterior probability of the latter's current value would decrease (increase, \resp).}

\begin{property}\label{prop:monotonicity}
A dialectical explanation kit $\set{\!\tuple{\ARR{a},\!\prop_{\ARR{a}}}, \!\tuple{\ARR{s},\!\prop_{\ARR{s}}}\!}$\footnote{Here $\ARR{a}$ 
and $\ARR{s}$ 
are some form of attack and support, \resp, depending on the specific explanation kit; e.g., for \emph{stochastically} dialectical explanation kits $\ARR{a}=\minusSTAT$ and $\ARR{s}=\plusSTAT$. } satisfies \emph{dialectical monotonicity} iff for any 
\FT{dialectical} IDX $\langle \!\Args_r,\! \mathcal{R}_a,\! \mathcal{R}_s \!\rangle$ \FT{ drawn from the kit (}for an\FT{y} explanandum $e \!\!\in\!\! \Args_r$
, input assignment $a \!\in\! \Inputs$\FT{)}, it holds that 
    for any $(\!x,\!y\!) \!\in\! \Rels_a\! \cup\! \Rels_s$, if $a' \!\!\in\! \Inputs$ is such 
     that $\SF(\!a'\!,\!x\!) \!\!\neq\!\! \SF(\!a,\!x\!)$ and $\SF(\!a'\!,\!z\!) \!\!=\!\! \SF(\!a,\!z\!)$ $\forall z \!\!\in\! \Influences(y) \!\!\setminus\!\! \{\! x \!\}$, then:
    
    \noindent $\bullet$ if $(x,y) \in \mathcal{R}_a$ then 
        $P(\SF(a, y) | a')
        >
        P(\SF(a, y) | a)$;
        
    \noindent $\bullet$ if $(x,y) \in \mathcal{R}_s$ then 
        $P(\SF(a, y) | a')
        <
        P(\SF(a, y) | a)$.
\end{property}

\begin{proposition}\label{thm:monotonicity}
\FT{Monotonically d}ialectical explanation kits satisfy dialectical monotonicity, while stochastically dialectical, LIME and SHAP explanation kits do not.
\end{proposition}

The final proposition gives the relationship between CF-IDXs and MD-/SD-IDXs.

\begin{proposition}\label{thm:CFXvsIDX}
    Given an MD-IDX $\langle \Args_r, \Atts, \Supps \rangle$, an SD-IDX $\langle \Args_r, \SDAtts, \SDSupps \rangle$ and a CF-IDX $\langle \Args'_r, \CIRels, \PIRels \rangle$ for 
    explanandum $e \in \Args_r$ 
    and input assignment $a \in \Inputs$, it holds that $\CIRels \subseteq \Supps$ and $\CIRels \subseteq \SDSupps$.
\end{proposition}

\EAB{Proposition \ref{thm:CFXvsIDX} shows that the \emph{critical influence} relation of CF-IDXs is a stronger variant of the relation of monotonic support. In fact it requires the current value of an influencee to change (rather then only the reduction of its probability)\ARX{, given a change in the value of the influencer}.}

\subsection{Empirical Analysis}\label{sec:empirical}

\FT{We used} several datasets
\FT{/}Bayesian networks \FT{(see Table~\ref{table:datasets} for details)}, for each of which we deployed a\ARX{n} 
\FT{NBC} (for single-label classification dataset) or a 
\FT{BCC} (for multi-label classification datasets and non-shallow Bayesian networks). When training BCs from datasets, we split them into \delold{a stratified }train and test set\EAB{s} (with 75/25\% ratio) and optimised the hyper-parameters \delold{of the model }using 5-fold cross-validation \delold{on the train set }(see 
\PBX{Appendix B for details}).

\begin{table}[t]
	\begin{small}
		\begin{center}
		\begin{tabular}{r|cc|cc|cc|cc}
  \multirow{ 2}{*}{\!\!Dataset\!\!} & 
  \multirow{ 2}{*}{\!\!\!BC$^{\dagger}$\!\!\!} & 
  \multirow{ 2}{*}{\!\!Size\!\!} & 
  \multicolumn{2}{c}{\!\!Variables\!\!} &
  \multicolumn{2}{c}{\!\!Types$^{\ddagger}$\!\!} & 
  \multicolumn{2}{c}{\!\!Performance$^{\mathsection}$\!\!}\\
  & & & \!\!$|\Obs|$\!\! & \!\!$|\Cla|$\!\! & \!\!$\Obs$\!\! & \!\!$\Cla$\!\! &\!\!Accuracy\!\!& \!\!F1\!\!\\
\hline
 \!\!\textbf{Shuttle}\tablefootnote{\label{footnote:UCI}Machine Learning Repository \cite{UCIDatasets}} & \!\!NBC\!\! & \!\!278\!\! & \!\!6\!\! & \!\!1\!\! & \!\!C\!\! & \!\!B\!\! & \!\!95.7\%\!\! & \!\!0.96\!\!\\
 \!\!\textbf{Votes}
 \textsuperscript{\ref{footnote:UCI}}
 & \!\!NBC\!\! & \!\!435\!\! & \!\!16\!\! & \!\!1\!\! & \!\!B\!\! & \!\!B\!\! & \!\!90.8\%\!\! & \!\!0.90\!\!\\
 \!\!\textbf{Parole}\tablefootnote{National Archive of Criminal Justice Data \cite{ParoleDataset}} & \!\!NBC\!\! & \!\!675\!\! & \!\!8\!\! & \!\!1\!\! & \!\!C\!\! & \!\!B\!\! & \!\!88.8\%\!\! & \!\!0.69\!\!\\
 \!\!\textbf{German}
 \textsuperscript{\ref{footnote:UCI}}
 & \!\!NBC\!\! & \!\!750\!\! & \!\!20\!\! & \!\!1\!\! & \!\!C\!\! & \!\!B\!\! & \!\!76.4\%\!\! & \!\!0.72\!\!\\
 \!\!\textbf{COMPAS}\tablefootnote{ProRepublica Data Store \cite{COMPASDataset}} & \!\!NBC\!\! & \!\!6951\!\! & \!\!12\!\! & \!\!1\!\! & \!\!C\!\! & \!\!B\!\! & \!\!70.5\%\!\! & \!\!0.71\!\!\\
 \!\!\textbf{Car}
 \textsuperscript{\ref{footnote:UCI}}
 & \!\!NBC\!\! & \!\!1728\!\! & \!\!6\!\! & \!\!1\!\! & \!\!C\!\! & \!\!C\!\! & \!\!86.6\%\!\! & \!\!0.76\!\!\\
 \!\!\textbf{Emotions}\tablefootnote{Multi-Label Classification Dataset Repository \cite{UCODatasets}} & \!\!BCC\!\! & \!\!593\!\! & \!\!72\!\! & \!\!6\!\! & \!\!C\!\! & \!\!B\!\! & \!\!80.2\%\!\! & \!\!0.70\!\!\\
  \!\!\textbf{Asia}\tablefootnote{\label{footnote:bnlearn}Bayesian Network Repository \cite{BNLearnDatasets}} & \!\!BCC\!\! & \!\!4\!\! & \!\!2\!\! & \!\!6\!\! & \!\!B\!\! & \!\!B\!\! & \!\!100\%\!\! & \!\!1.00\!\!\\
 \!\!\textbf{Child}
 \textsuperscript{\ref{footnote:bnlearn}}
 & \!\!BCC\!\! & \!\!1080\!\! & \!\!7\!\! & \!\!13\!\! & \!\!C\!\! & \!\!C\!\! & \!\!80.6\%\!\! & \!\!0.66\!\!\\
\end{tabular}
		\end{center}
    	\protect\caption{Characteristics of the datasets used in the evaluation.
    	($\dagger$) NBC ({N}aive 
    	\FT{BC}) or BCC ({B}ayesian {C}hain {C}lassifier); ($\ddagger$) \textbf{B}inary or \textbf{C}ategorical; ($\mathsection$) accuracy and macro F1 score on the {test} set, averaged for multi-label settings.}\label{table:datasets}
    \end{small}
\end{table}

\textbf{Computational cost.}
\EAB{MD-IDX\ARX{s} and SD-IDX\ARX{s} are model-specific explanations and can thus access the internal state of the 
\ARX{BC}. This allows them to be more efficient in extracting information than model-agnostic methods, e.g., methods based on the sampling of the input space.}
Formally, let $t_p$ be the time to compute a prediction and its associated posterior probabilities \FTx{(}in our experiments\footnote{We used a machine with 
 \textit{Intel i9-9900X} at $3.5 Ghz$ and $32 GB$ of RAM with no GPU acceleration. 
\FT{For BCCs we did not use production-ready code optimized for performances.}}, $t_p$ ranged from $3 \mu s$ for the simplest NBC to $40 ms$ for the more complex BCC\FTx{)}. 
The time complexity to compute 
\FT{whether an} influence $(x,y)$ \FT{belongs to a relation in} MD-IDXs or \ARX{i}n SD-IDXs
\EAX{The time complexity to compute 
\FT{whether an} influence $(x,y)$ \FT{belongs to a relation in} MD-IDXs or an SD-IDXs, denoted as $T_{1-IDX}$, is a function of the number of values that $x$ can be assigned to (i.e, $|\Val(x)|$) because in order to assess that, the algorithm has to check how the posterior probability of $y$ changes when changing $x$ (as \ARX{per} 
Defs.~\ref{def:DIDX} and \ref{def:SDIDX}).}

\(
    T_{1-IDX}(\Val(x)) \!\!=\!\! \Theta \left( t_p \cdot [1 \!+\! |\Val(x)| - 1] \right) \!\!=\!\! \Theta \left( t_p \cdot |\Val(x)| \right)
\).
\\
\EAX{The time complexity to compute a full MD/SD-IDX, denoted as $T_{IDX}$, is the sum of the above for all the influences, with some potential savings that can be achieved depending on the specific type of BC that is being used (e.g., in BCC we can decompose the prediction into multiple NBCs, therefore leading to a lower cost).} Thus\EAX{, formally, }the time complexity to \EAX{compute} 
\EAX{an }MD-IDX \EAX{or an }SD-IDX is:

\(
    T_{IDX}(\Influences, \Val) = \Theta\left(
    t_p \cdot \left[1 +  \sum_{x \in \{x | (x,y) \in \Influences \}} (|\Val(x)| - 1) \right]
    \right)
\).
\\
The \EAX{asymptotic complexity can be further simplified to}:

\(
    T_{IDX}(\Val) = \Theta \left(
        t_p \cdot \sum_{x \in \Args} |\Val(x)|
    \right)
\)
\\
\FT{which} is \emph{linear} with respect to the \EAB{sum of the }number of variables' values
. This makes our dialectical explanations \emph{efficient}. \EAB{As a comparison, the time taken to generate an MD-IDX for the \textit{Child BC} is less than $60 \cdot t_p$ while the time taken to generate a LIME explanation with default parameters is $5000 \cdot t_p$ because LIME \EAX{asks the BC for the prediction of} a sample of 5000 inputs before generating an explanation.}

\textbf{Stability.} 
\EAB{All 
\ARX{IDXs}
are \emph{unique} (for each input assignment-explanandum pair), but, u}nlike LIME and SHAP, MD-IDXs, SD-IDXs \FT{and CF-IDXs} use deterministic methods \EAB{(i.e., not based on sampling)} 
\FT{and thus} \EAB{their explanations} are \EAB{also} \emph{stable} \ARX{\cite{Sokol_20}}.

\begin{table*}[t]
	\begin{small}
		\begin{center}
		   \begin{tabular}{r|ccc|ccc|ccc|cc|cc}
\multirow{2}{*}{Dataset} & \multicolumn{3}{c}{SD-IDX} & \multicolumn{3}{c}{MD-IDX}  & \multicolumn{3}{c}{CF-IDX} & \multicolumn{2}{c}{LIME-IDX} & \multicolumn{2}{c}{SHAP-IDX}\\
                       & \!\!$\Rels_{\dot{+}}$\!\! & \!\!$\Rels_{\dot{-}}$\!\! & \!\!$\Rels_{\dot{-}\dot{+}}^{\Cla}$\!\! & \!\!$\Rels_{+}$\!\! & \!\!$\Rels_{-}$\!\! & \!\!$\Rels_{-+}^{\Cla}$\!\! & \!\!$\Rels_!$\!\! & \!\!$\Rels_*$\!\! & \!\!$\Rels_{*!}^{\Cla}$\!\! & \!\!$\SuppsLIME$\!\! & \!\!$\AttsLIME$\!\! & \!\!$\SuppsSHAP$\!\! & \!\!$\AttsSHAP$\!\! \\
\hline
  \!\textbf{Shuttle}\! &                 \!59.0\%\! &                 \!41.0\%\! &                            \!$\times$\! &           \!51.2\%\! &           \!32.4\%\! &                \!$\times$\! &         \!17.9\%\! &         \!46.0\%\! &                \!$\times$\! &         \!59.8\%\! &         \!40.2\%\! &         \!61.9\%\! &         \!38.1\%\! \\ 
    \!\textbf{Votes}\! &                 \!77.1\%\! &                 \!22.9\%\! &                            \!$\times$\! &           \!77.1\%\! &           \!22.9\%\! &                \!$\times$\! &          \!3.0\%\! &         \!74.1\%\! &                \!$\times$\! &         \!77.1\%\! &         \!22.9\%\! &         \!73.2\%\! &          \!7.3\%\! \\
   \!\textbf{Parole}\! &                 \!61.5\%\! &                 \!38.5\%\! &                            \!$\times$\! &           \!38.5\%\! &           \!27.4\%\! &                \!$\times$\! &          \!2.1\%\! &         \!70.5\%\! &                \!$\times$\! &         \!55.6\%\! &         \!44.4\%\! &         \!56.4\%\! &         \!43.6\%\! \\
   \!\textbf{German}\! &                 \!59.3\%\! &                 \!40.7\%\! &                            \!$\times$\! &           \!29.6\%\! &           \!22.0\%\! &                \!$\times$\! &      \!$\times$\! &      \!$\times$\! &                \!$\times$\! &         \!55.9\%\! &         \!44.1\%\! &         \!46.9\%\! &         \!36.4\%\! \\
   \!\textbf{COMPAS}\! &                 \!67.0\%\! &                 \!33.0\%\! &                            \!$\times$\! &           \!45.4\%\! &           \!20.3\%\! &                \!$\times$\! &      \!$\times$\! &      \!$\times$\! &                \!$\times$\! &         \!65.7\%\! &         \!34.3\%\! &         \!35.6\%\! &         \!19.1\%\! \\
      \!\textbf{Car}\! &                 \!57.4\%\! &                 \!42.6\%\! &                            \!$\times$\! &           \!39.3\%\! &           \!21.1\%\! &                \!$\times$\! &         \!14.2\%\! &         \!66.4\%\! &                \!$\times$\! &         \!55.3\%\! &         \!44.7\%\! &         \!55.9\%\! &         \!44.1\%\! \\
 \!\textbf{Emotions}\! &                 \!56.9\%\! &                 \!24.0\%\! &                                \!1.1\%\! &           \!10.3\%\! &            \!5.4\%\! &                    \!1.1\%\! &          \!1.9\%\! &          \!1.5\%\! &                \!$\times$\! &         \!60.6\%\! &         \!39.4\%\! &         \!56.8\%\! &         \!10.3\%\! \\
    \!\textbf{Child}\! &                 \!77.5\%\! &                 \!22.5\%\! &                               \!64.0\%\! &           \!65.4\%\! &           \!15.1\%\! &                   \!64.0\%\! &         \!34.6\%\! &         \!27.6\%\! &                   \!64.0\%\! &         \!54.0\%\! &         \!41.3\%\! &         \!24.4\%\! &          \!9.7\%\! \\
     \!\textbf{Asia}\! &                 \!87.5\%\! &                 \!12.5\%\! &                               \!62.5\%\! &           \!87.5\%\! &           \!12.5\%\! &                   \!62.5\%\! &         \!46.9\%\! &          \!9.4\%\! &                   \!62.5\%\! &         \!70.8\%\! &         \!29.2\%\! &         \!54.2\%\! &         \!20.8\%\! \\
\end{tabular}
        \end{center}
    	\protect\caption{Average percentage of the influences that \AR{constitute} 
    	relation\FT{s} 
    	for several 
    	\FT{instantiated IDXs}. 
    	\FT{For any relation type $t$, } $\Rels^{\Cla}_{
    	\FT{t}} = \{(x,y) \in 
    	\FT{\Rels_t} | x, y \in \Cla \}$. $\times$ indicates results that cannot be computed in these settings or that \AR{must be} 
    	0 due to the BC type. 
    	\EAB{Note that percentages may not sum to $100\%$ because not all the influences are categorised into relations.}}\label{table:prevalence_of_relations}
    \end{small}
\end{table*}

\begin{figure}[ht!]
\centering
    \includegraphics[width=\linewidth]{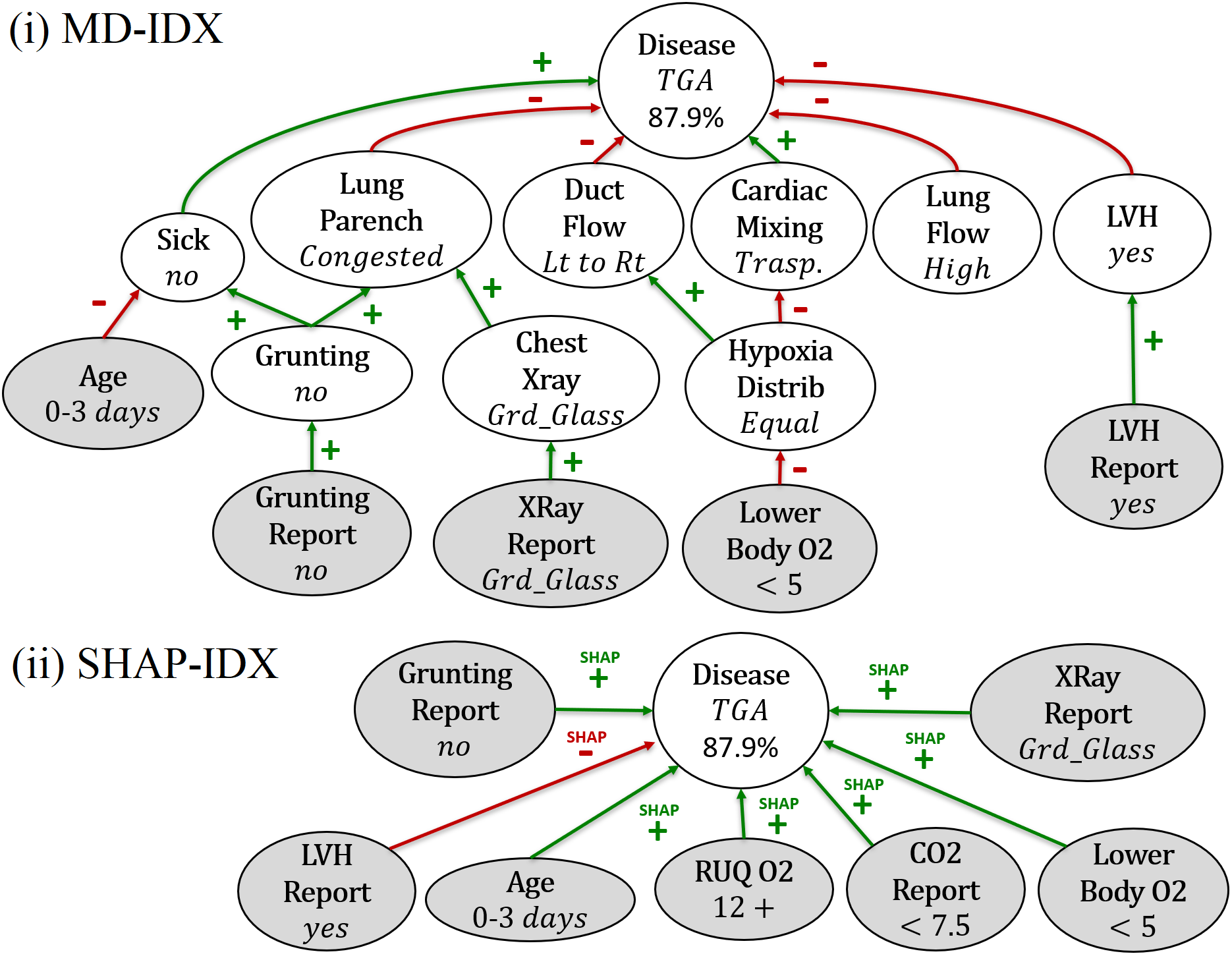}
    \protect\caption{Example MD-IDX (i) and SHAP-IDX (ii), in graphical form, for \ARR{explanandum} \textit{Disease} in the \textit{Child} BCC.  
    Each node represents a variable with 
    the assigned/estimated value in italics
    . Grey and white nodes indicate observations and classifications, \resp\  $+$/${\protect\plusSHAP}$ and $-$/${\protect\minusSHAP}$ indicate supports and attacks, \resp\ Note that other ways to present IDXs to users may be suitable in some settings: this is left for future work.}
    \label{fig:IDXeg}
\end{figure}

\textbf{Size of the explanations.} \EAB{In order to \emph{understand how many influences translated into relations and their prevalence} w}e 
\EAB{calculated the number of} relations of \AR{each of the \FT{instantiated} IDXs from Section \ref{sec:IDX}}; the results are reported in Table~\ref{table:prevalence_of_relations}\EAB{\ARX{, showing} the percentage of influences that are categorised into each relation}. 
We 
\FT{note} that: 
\textbf{(1)}
MD-IDXs, SD-IDXs \EA{and CF-IDX} are non-shallow, thus, 
\FT{when non-naive BCs are}
used, they also find relationships between pairs of classifications (see $\Rels_{-+}^{\Cla}$, $\Rels_{\minusSTAT\plusSTAT}^{\Cla}$ \EA{and $\Rels_{*!}^{\Cla}$} in Table~\ref{table:prevalence_of_relations}) that, depending on the structure of the Bayesian network underlying the BC, can be also numerous, e.g., for 
\FT{\textit{Child} and \textit{Asia}}\EAB{; this shows that our notion of influences provide a deeper insight into the intermediate variables of the model, that are otherwise neglected by \emph{input-output influences}};
\textbf{(2)} 
SD-IDXs and LIME-IDXs tend to be\FT{have} similar\FT{ly here}, while MD-IDXs 
\EA{tend to \FT{include fewer influences} 
than SD-IDXs 
(in line with Proposition \ref{thm:DIDXtoSDIDX}), which in turn 
\FT{include fewer influences}
than critical influences in CF-IDXs (in line with Proposition \ref{thm:CFXvsIDX})}; note that, CF-IDXs could not be computed in all settings because their computational complexity is \textit{exponential} \EA{\wrt\ the cardinality of the 
variables\FT{' sets of values}, while the complexity of SM-IDXs and MD-IDXs is \textit{linear} (
\FT{see \PBX{above}})}; 
\textbf{(3)} 
in some settings, SHAP-IDXs 
\AR{fail} to capture the majority of 
attacks captured by the other dialectical explanations (e.g., 
\FT{for} \textit{Votes} and \textit{Emotions}
).

\begin{table*}[t]
	\begin{small}
		\begin{center}

\begin{tabular}{r|cc|cc|cc|cc}
\multirow{2}{*}{Dataset} &     
\multicolumn{2}{c}{MD-/LIME-IDX} &   
\multicolumn{2}{c}{MD-/SHAP-IDX} &
\multicolumn{2}{c}{SD-/LIME-IDX} &  
\multicolumn{2}{c}{MD-/SHAP-IDX} \\
                         &  
                         \!\!Agree\!\!  &
                         \!\!Disagree\!\!  &
                         \!\!Agree\!\!  &
                         \!\!Disagree\!\!  &
                         \!\!Agree\!\!  &
                         \!\!Disagree\!\!  &
                         \!\!Agree\!\!  &
                         \!\!Disagree\!\!  \\
\hline
\textbf{Shuttle}  & 
\!\!82.4\%\!\! & 
\!\!17.6\%\!\! & 
\!\!83.6\%\!\! & 
\!\!16.4\%\!\! &
\!\!98.8\%\!\! &  
\!\!1.2\%\!\! & 
\!\!97.1\%\!\! &  
\!\!2.9\%\!\! \\
\textbf{Votes}    & 
\!\!99.1\%\!\! &  
\!\!0.9\%\!\! & 
\!\!78.8\%\!\! & 
\!\!21.2\%\!\! &
\!\!99.1\%\!\! &  
\!\!0.9\%\!\! & 
\!\!78.8\%\!\! & 
\!\!21.2\%\!\! \\
\textbf{Parole}   & 
\!\!65.8\%\!\! & 
\!\!34.2\%\!\! & 
\!\!65.8\%\!\! & 
\!\!34.2\%\!\! &
\!\!91.1\%\!\! &  
\!\!8.9\%\!\! & 
\!\!94.9\%\!\! &  
\!\!5.1\%\!\! \\
\textbf{German}   & 
\!\!50.1\%\!\! & 
\!\!49.9\%\!\! & 
\!\!38.2\%\!\! & 
\!\!61.8\%\!\! &
\!\!83.6\%\!\! & 
\!\!16.4\%\!\! & 
\!\!73.5\%\!\! & 
\!\!26.5\%\!\! \\
\textbf{COMPAS}   & 
\!\!65.7\%\!\! & 
\!\!34.3\%\!\! & 
\!\!20.4\%\!\! & 
\!\!79.6\%\!\! &
\!\!98.3\%\!\! &  
\!\!1.7\%\!\! & 
\!\!54.1\%\!\! & 
\!\!45.9\%\!\! \\
\textbf{Car}      & 
\!\!59.1\%\!\! & 
\!\!40.9\%\!\! & 
\!\!60.3\%\!\! & 
\!\!39.7\%\!\! &
\!\!94.7\%\!\! &  
\!\!5.3\%\!\! & 
\!\!97.1\%\!\! &  
\!\!2.9\%\!\! \\
\textbf{Emotions} & 
\!\!15.0\%\!\! & 
\!\!85.0\%\!\! & 
\!\!12.3\%\!\! & 
\!\!87.7\%\!\! &
\!\!65.0\%\!\! & 
\!\!35.0\%\!\! & 
\!\!51.9\%\!\! & 
\!\!48.1\%\!\! \\
\textbf{Child}    &  
\!\!8.3\%\!\! & 
\!\!91.7\%\!\! &  
\!\!8.3\%\!\! & 
\!\!91.7\%\!\! &
\!\!9.2\%\!\! & 
\!\!90.8\%\!\! &  
\!\!9.3\%\!\! & 
\!\!90.7\%\!\! \\
\textbf{Asia}     & 
\!\!25.0\%\!\! & 
\!\!75.0\%\!\! & 
\!\!25.0\%\!\! & 
\!\!75.0\%\!\! &
\!\!25.0\%\!\! & 
\!\!75.0\%\!\! & 
\!\!25.0\%\!\! & 
\!\!75.0\%\!\! 
\end{tabular}

        \end{center}
    	\protect\caption{
    	\FT{(Dis)Agreement between MD-IDXs\ARX{/}\EAB{SD-IDXs} and} LIME-IDXs\EAB{/}SHAP-IDXs. 
    	Agreement means 
    	identifying the same (possibly empty) relations 
    	for the same influences\EAB{, in practice ${|\Rels_{\cdot}|}/{|I|}$, where $\cdot$ is the relation type and, accordingly with Defs.~\ref{def:Influences}~and~\ref{def:IOInfluences}, $I=\Influences$ for SD-/MD-/CF-IDX and $I=\Influences_{io}$ for LIME-/SHAP-IDX.}}\label{table:overlap}
    \end{small} 
\end{table*}

\textbf{Agreement of Dialectical Explanations.} \EAB{To show \emph{how our explanations differ from existing dialectical explanations} w}e compared the 
relations 
\FT{in} MD-IDXs/SD-IDXs with those \FT{in} LIME-IDXs/SHAP-IDXs, analysing how often they agree \EAB{in }
identify\EAB{ing} attack\EAB{s} 
\EAB{or} support\EAB{s} 
between \EAB{observations and classifications}
. \FT{Table~\ref{table:overlap} shows the results \EAB{in percentage for each pair, e.g, between MD-IDX and LIME-IDX $|(\Atts \cap \AttsLIME) \cup (\Supps \cap \SuppsLIME)|/|\Influences_{io}|$}
. We note that: }
\textbf{(1)}
MD-IDXs agree on average $52.30\%$ and $43.6\%$ of the \ARX{time} while SD-IDXs agree on average $73.9\%$ and $64.63\%$ of \ARX{the time} with LIME-IDXs and SHAP-IDXs, respectively
, due to MD-IDXs' stronger constraints on the selection of attacks and supports; 
\textbf{(2)}
\FT{when} a BCC with many classifications is used (as in \textit{Child}, \textit{Asia}, \textit{Emotions}), the agreement decreases considerably,  due to 
LIME-IDXs and SHAP-IDXs 
\FT{being shallow, and thus} 
\FT{selecting} different 
\FT{influences from non-shallow MD-IDXs and SD-IDXs,} as described in Section~\ref{sec:otherdialecticalIDX} \EA{and exemplified by Figure~\ref{fig:IDXeg}}.

\begin{table}[t]
\begin{small}
\begin{center}
\begin{tabular}{r|c|c|c}
Dataset & SM-IDX & LIME-IDX & SHAP-IDX\\
\hline
\textbf{Shuttle}  &   \!\!5.8\%\!\! &   \!\!6.8\%\!\! &   \!\!6.5\%\!\! \\
\textbf{Votes}    &   \!\!0.0\%\!\! &   \!\!0.2\%\!\! &   \!\!0.1\%\!\! \\
\textbf{Parole}   &  \!\!13.3\%\!\! &  \!\!13.2\%\!\! &  \!\!13.8\%\!\! \\
\textbf{German}   &  \!\!18.5\%\!\! &  \!\!20.8\%\!\! &  \!\!19.8\%\!\! \\
\textbf{COMPAS}   &  \!\!12.3\%\!\! &  \!\!12.5\%\!\! &  \!\!22.7\%\!\! \\
\textbf{Car}      &  \!\!16.8\%\!\! &  \!\!18.5\%\!\! &  \!\!17.4\%\!\! \\
\textbf{Emotions} &  \!\!12.0\%\!\! &  \!\!11.9\%\!\! &   \!\!8.9\%\!\! \\
\textbf{Child}    &   \!\!7.1\%\!\! &   \!\!2.5\%\!\! &   \!\!5.6\%\!\! \\
\textbf{Asia}     &   \!\!0.0\%\!\! &   \!\!0.0\%\!\! &   \!\!0.0\%\!\! \\

\end{tabular}
\end{center}
\protect\caption{Average percentage of \FT{influences in} relations \textit{not} satisfying the dialectical monotonicity property, obtained with a sample of 25,000 \FT{influences in} relations 
for 250 \EA{data-points}.}\label{table:monotonicity}
\end{small}
\end{table}

\textbf{\ARX{Satisfaction of} Dialectical Monotonicity
.} \EA{
\FT{We} checked the number of \FT{influences in} relations \AR{in CF-IDXs, LIME-IDXs and SHAP-IDXs which do not satisfy \emph{dialectical monotonicity} 
\ARR{(}Property \ref{prop:monotonicity}\ARR{),} 
\FT{\ARR{which holds} for MD-IDXs}}. \FT{The results are in}  Table~\ref{table:monotonicity}. \FT{We note} that:
\textbf{(1)}
SM-IDXs violate the dialectical monotonicity constraint significantly ($p < 0.05$) less than other methods for all the NBCs
, while their increase in the number of violations of dialectical monotonicity for BCCs 
is due to SM-IDXs \FT{being} 
non-shallow
, unlike 
LIME-IDXs and SHAP-IDXs
; 
\textbf{(2)} 
all \AR{three} methods violate the dialectical monotonicity constraint.}
The violation of 
dialectical monotonicity 
gives rise to counter-intuitive results from a dialectical perspective.
Consider the example in Figure \ref{fig:egBC}: 
if \emph{raining} is considered to be an output, LIME draws a negative contribution from \emph{windy} to \emph{raining}, which 
\FT{makes little sense} given the structure of the BC (this also indicates the benefits of incorporating a model's structure into IDXs via influences).
Further, for the SHAP-IDX in Figure \ref{fig:IDXeg}ii, 
\AR{changing the value of \textit{Disease}'s supporter, \textit{Age}, \FT{so that it is no longer a supporter,}}
results in the probability of \textit{Disease} being \textit{TGA} increasing.
\section{Conclusions}
\label{sec:conc}

We have introduced IDXs, a general formalism for producing various forms of 
explanations for BCs.
We have demonstrated IDXs' versatility by 
\FT{giving novel forms of dialectical explanations as IDX instances}, based on the principle of monotonicity, 
and integrating existing 
\FT{notions of explanation}
.
We have 
performed a wide-ranging evaluation 
\FT{with} theoretical \ARX{and} empirical 
analyses of IDXs 
with respect to existing methods, 
identifying some 
\FT{
advantages of IDXs with non-naive BCs}
.

\delold{Given that IDXs offer a new perspective on explanation methods, providing a deeper, model-specific representation of a BC, there are numerous directions we would like to explore for future work. A first step could be exploring the aforementioned hybrid of attribution methods and the deeper structure afforded by IDXs, while exploring the general concept of IDXs for other AI methods could also be fruitful.
Another important line of work would be the extension of the user studies to investigate how effectively our explanations function \wrt\ a user's experience, e.g., experimenting with different styles and formats of IDX.} 

\section*{Acknowledgements}
This research was funded in part by J.P. Morgan and  by the Royal Academy of Engineering under the Research Chairs and Senior Research Fellowships scheme.  Any views or opinions expressed herein are solely those of the authors listed, and may differ from the views and opinions expressed by J.P. Morgan or its affiliates. This material is not a product of the Research Department of J.P. Morgan Securities LLC. This material should not be construed as an individual recommendation for any particular client and is not intended as a recommendation of particular securities, financial instruments or strategies for a particular client.  This material does not constitute a solicitation or offer in any jurisdiction.

\bibliographystyle{named}
\bibliography{bib}

\begin{thebibliography}{}

\bibitem[\protect\citeauthoryear{Albini \bgroup \em et al.\egroup }{2020}]{BC}
Emanuele Albini, Antonio Rago, Pietro Baroni, and Francesca Toni.
\newblock Relation-based counterfactual explanations for bayesian network
  classifiers.
\newblock In {\em Proceedings of the Twenty-Ninth Int. Joint Conf. on
  Artificial Intelligence, {IJCAI} 2020}, pages 451--457, 2020.

\bibitem[\protect\citeauthoryear{Bach \bgroup \em et al.\egroup
  }{2015}]{Bach_15}
Sebastian Bach, Alexander Binder, Grégoire Montavon, Frederick Klauschen,
  Klaus-Robert Müller, and Wojciech Samek.
\newblock On pixel-wise explanations for non-linear classifier decisions by
  layer-wise relevance propagation.
\newblock {\em PLOS ONE}, 10(7):1--46, 07 2015.

\bibitem[\protect\citeauthoryear{Bau \bgroup \em et al.\egroup }{2017}]{Bau_17}
David Bau, Bolei Zhou, Aditya Khosla, Aude Oliva, and Antonio Torralba.
\newblock Network dissection: Quantifying interpretability of deep visual
  representations.
\newblock In {\em 2017 {IEEE} Conference on Computer Vision and Pattern
  Recognition, {CVPR} 2017}, pages 3319--3327, 2017.

\bibitem[\protect\citeauthoryear{BNlearn}{2020}]{BNLearnDatasets}
BNlearn.
\newblock Bayesian network repository - an r package for bayesian network
  learning and inference, 2020.

\bibitem[\protect\citeauthoryear{Cayrol and Lagasquie-Schiex}{2005}]{Cayrol:05}
Claudette Cayrol and Marie-Christine Lagasquie-Schiex.
\newblock On the acceptability of arguments in bipolar argumentation
  frameworks.
\newblock In {\em Proceedings of the 8th European Conf. on Symbolic and
  Quantitative Approaches to Reasoning with Uncertainty}, pages 378--389.
  Springer Berlin Heidelberg, 2005.

\bibitem[\protect\citeauthoryear{Cocarascu \bgroup \em et al.\egroup
  }{2019}]{RT}
Oana Cocarascu, Antonio Rago, and Francesca Toni.
\newblock Extracting dialogical explanations for review aggregations with
  argumentative dialogical agents.
\newblock In {\em Proceedings of the 18th Int. Conf. on Autonomous Agents and
  MultiAgent Systems, {AAMAS} '19}, pages 1261--1269, 2019.

\bibitem[\protect\citeauthoryear{Cyras \bgroup \em et al.\egroup
  }{2019}]{Cyras_19}
Kristijonas Cyras, Dimitrios Letsios, Ruth Misener, and Francesca Toni.
\newblock Argumentation for explainable scheduling.
\newblock In {\em The Thirty-Third {AAAI} Conf. on Artificial Intelligence,
  {AAAI} 2019}, pages 2752--2759, 2019.

\bibitem[\protect\citeauthoryear{Darwiche and Hirth}{2020}]{Darwiche_20}
Adnan Darwiche and Auguste Hirth.
\newblock On the reasons behind decisions.
\newblock In {\em 24th European Conference on Artificial Intelligence ({ECAI}
  2020)}, 2020.

\bibitem[\protect\citeauthoryear{Dung}{1995}]{Dung_95}
Phan~Minh Dung.
\newblock {On the Acceptability of Arguments and its Fundamental Role in
  Nonmonotonic Reasoning, Logic Programming and n-Person Games}.
\newblock {\em Artificial Intelligence}, 77(2):321--358, 1995.

\bibitem[\protect\citeauthoryear{{Enrique Sucar} \bgroup \em et al.\egroup
  }{2014}]{ChainClassifiers}
L.~{Enrique Sucar}, Concha Bielza, Eduardo~F. Morales, Pablo Hernandez-Leal,
  Julio~H. Zaragoza, and Pedro Larrañaga.
\newblock Multi-label classification with bayesian network-based chain
  classifiers.
\newblock {\em Pattern Recognition Letters}, 41:14 -- 22, 2014.

\bibitem[\protect\citeauthoryear{Fan and Toni}{2015}]{Fan_15}
Xiuyi Fan and Francesca Toni.
\newblock On computing explanations in argumentation.
\newblock In {\em Proceedings of the Twenty-Ninth {AAAI} Conf. on Artificial
  Intelligence}, pages 1496--1502, 2015.

\bibitem[\protect\citeauthoryear{Friedman \bgroup \em et al.\egroup
  }{1997}]{Friedman_97}
Nir Friedman, Dan Geiger, and Mois{\'{e}}s Goldszmidt.
\newblock Bayesian network classifiers.
\newblock {\em Mach. Learn.}, 29(2-3):131--163, 1997.

\bibitem[\protect\citeauthoryear{Garc{\'{\i}}a \bgroup \em et al.\egroup
  }{2013}]{Garcia_13}
Alejandro~Javier Garc{\'{\i}}a, Carlos~Iv{\'{a}}n Ches{\~{n}}evar,
  Nicol{\'{a}}s~D. Rotstein, and Guillermo~Ricardo Simari.
\newblock Formalizing dialectical explanation support for argument-based
  reasoning in knowledge-based systems.
\newblock {\em Expert Syst. Appl.}, 40(8):3233--3247, 2013.

\bibitem[\protect\citeauthoryear{Guidotti \bgroup \em et al.\egroup
  }{2019}]{guidotti}
Riccardo Guidotti, Anna Monreale, Salvatore Ruggieri, Franco Turini, Fosca
  Giannotti, and Dino Pedreschi.
\newblock A survey of methods for explaining black box models.
\newblock {\em {ACM} Comput. Surv.}, 51(5):93:1--93:42, 2019.

\bibitem[\protect\citeauthoryear{Halpern and Pearl}{2001a}]{Halpern_01_IJCAI}
Joseph~Y. Halpern and Judea Pearl.
\newblock Causes and explanations: {A} structural-model approach - part {II:}
  explanations.
\newblock In {\em Proceedings of the Seventeenth Int. Joint Conf. on Artificial
  Intelligence, {IJCAI} 2001}, pages 27--34, 2001.

\bibitem[\protect\citeauthoryear{Halpern and Pearl}{2001b}]{Halpern_01_UAI}
Joseph~Y. Halpern and Judea Pearl.
\newblock Causes and explanations: {A} structural-model approach: Part 1:
  Causes.
\newblock In {\em {UAI} '01: Proceedings of the 17th Conf. in Uncertainty in
  Artificial Intelligence}, pages 194--202, 2001.

\bibitem[\protect\citeauthoryear{Ignatiev \bgroup \em et al.\egroup
  }{2019a}]{Ignatiev_19_AAAI}
Alexey Ignatiev, Nina Narodytska, and Jo{\~{a}}o Marques{-}Silva.
\newblock Abduction-based explanations for machine learning models.
\newblock In {\em The Thirty-Third {AAAI} Conf. on Artificial Intelligence,
  {AAAI} 2019}, pages 1511--1519, 2019.

\bibitem[\protect\citeauthoryear{Ignatiev \bgroup \em et al.\egroup
  }{2019b}]{Ignatiev_19}
Alexey Ignatiev, Nina Narodytska, and Jo{\~{a}}o Marques{-}Silva.
\newblock On relating explanations and adversarial examples.
\newblock In {\em Advances in Neural Information Processing Systems 32: Annual
  Conf. on Neural Information Processing Systems 2019, NeurIPS 2019}, pages
  15857--15867, 2019.

\bibitem[\protect\citeauthoryear{Ignatiev}{2020}]{Ignatiev_20}
Alexey Ignatiev.
\newblock Towards trustable explainable {AI}.
\newblock In {\em Proceedings of the Twenty-Ninth Int. Joint Conf. on
  Artificial Intelligence, {IJCAI} 2020}, pages 5154--5158, 2020.

\bibitem[\protect\citeauthoryear{Koller and Sahami}{1996}]{Koller_96}
Daphne Koller and Mehran Sahami.
\newblock Toward optimal feature selection.
\newblock In {\em Machine Learning, Proceedings of the Thirteenth Int. Conf.
  {(ICML} '96)}, pages 284--292, 1996.

\bibitem[\protect\citeauthoryear{Lundberg and Lee}{2017}]{Lundberg_17}
Scott~M. Lundberg and Su{-}In Lee.
\newblock A unified approach to interpreting model predictions.
\newblock In {\em Advances in Neural Information Processing Systems 30: Annual
  Conf. on Neural Information Processing Systems 2017}, pages 4765--4774, 2017.

\bibitem[\protect\citeauthoryear{Miller}{2019}]{Miller_19}
Tim Miller.
\newblock Explanation in artificial intelligence: Insights from the social
  sciences.
\newblock {\em Artificial Intelligence}, 267:1--38, 2019.

\bibitem[\protect\citeauthoryear{Moyano}{2020}]{UCODatasets}
Jose~M. Moyano.
\newblock Multi-label classification dataset repository, 2020.

\bibitem[\protect\citeauthoryear{NACJD}{2004}]{ParoleDataset}
{National Archive of Criminal Justice Data} NACJD.
\newblock National corrections reporting program, 2004.

\bibitem[\protect\citeauthoryear{Naveed \bgroup \em et al.\egroup
  }{2018}]{Naveed_18}
Sidra Naveed, Tim Donkers, and J{\"{u}}rgen Ziegler.
\newblock Argumentation-based explanations in recommender systems: Conceptual
  framework and empirical results.
\newblock In {\em Adjunct Publication of the 26th Conf. on User Modeling,
  Adaptation and Personalization, {UMAP} 2018}, pages 293--298, 2018.

\bibitem[\protect\citeauthoryear{Nielsen \bgroup \em et al.\egroup
  }{2008}]{Nielsen_08}
Ulf~H. Nielsen, Jean{-}Philippe Pellet, and Andr{\'{e}} Elisseeff.
\newblock Explanation trees for causal bayesian networks.
\newblock In {\em {UAI} 2008, Proceedings of the 24th Conf. in Uncertainty in
  Artificial Intelligence}, pages 427--434, 2008.

\bibitem[\protect\citeauthoryear{Olah \bgroup \em et al.\egroup
  }{2018}]{Olah_18}
Chris Olah, Arvind Satyanarayan, Ian Johnson, Shan Carter, Ludwig Schubert,
  Katherine Ye, and Alexander Mordvintsev.
\newblock The building blocks of interpretability.
\newblock {\em Distill}, 3(3):e10, 2018.

\bibitem[\protect\citeauthoryear{Poyiadzi \bgroup \em et al.\egroup
  }{2020}]{Poyiadzi_20}
Rafael Poyiadzi, Kacper Sokol, Ra{\'{u}}l Santos{-}Rodr{\'{\i}}guez, Tijl~De
  Bie, and Peter~A. Flach.
\newblock {FACE:} feasible and actionable counterfactual explanations.
\newblock In {\em {AIES} '20: {AAAI/ACM} Conf. on AI, Ethics, and Society},
  pages 344--350, 2020.

\bibitem[\protect\citeauthoryear{ProPublica}{2016}]{COMPASDataset}
Data~Store ProPublica.
\newblock Compas recidivism risk score data and analysis, 2016.

\bibitem[\protect\citeauthoryear{Rago \bgroup \em et al.\egroup }{2018}]{rec}
Antonio Rago, Oana Cocarascu, and Francesca Toni.
\newblock Argumentation-based recommendations: Fantastic explanations and how
  to find them.
\newblock In {\em Proceedings of the Twenty-Seventh Int. Joint Conf. on
  Artificial Intelligence, {IJCAI} 2018}, pages 1949--1955, 2018.

\bibitem[\protect\citeauthoryear{Ribeiro \bgroup \em et al.\egroup
  }{2016}]{Ribeiro_16}
Marco~T{\'{u}}lio Ribeiro, Sameer Singh, and Carlos Guestrin.
\newblock ``{W}hy should {I} trust you?": Explaining the predictions of any
  classifier.
\newblock In {\em Proceedings of the 22nd {ACM} {SIGKDD} Int. Conf. on
  Knowledge Discovery and Data Mining}, pages 1135--1144, 2016.

\bibitem[\protect\citeauthoryear{Ribeiro \bgroup \em et al.\egroup
  }{2018}]{Ribeiro_18}
Marco~T{\'{u}}lio Ribeiro, Sameer Singh, and Carlos Guestrin.
\newblock Anchors: High-precision model-agnostic explanations.
\newblock In {\em Proceedings of the Thirty-Second {AAAI} Conf. on Artificial
  Intelligence (AAAI-18)}, pages 1527--1535, 2018.

\bibitem[\protect\citeauthoryear{Schwab and Karlen}{2019}]{Schwab_19}
Patrick Schwab and Walter Karlen.
\newblock Cxplain: Causal explanations for model interpretation under
  uncertainty.
\newblock In {\em Advances in Neural Information Processing Systems 32: Annual
  Conf. on Neural Information Processing Systems 2019, NeurIPS 2019}, pages
  10220--10230, 2019.

\bibitem[\protect\citeauthoryear{Sharma \bgroup \em et al.\egroup
  }{2020}]{Sharma_20}
Shubham Sharma, Jette Henderson, and Joydeep Ghosh.
\newblock {CERTIFAI:} {A} common framework to provide explanations and analyse
  the fairness and robustness of black-box models.
\newblock In {\em {AIES} '20: {AAAI/ACM} Conf. on AI, Ethics, and Society},
  pages 166--172, 2020.

\bibitem[\protect\citeauthoryear{Shih \bgroup \em et al.\egroup
  }{2018}]{Shih_18}
Andy Shih, Arthur Choi, and Adnan Darwiche.
\newblock A symbolic approach to explaining bayesian network classifiers.
\newblock In {\em Proceedings of the Twenty-Seventh Int. Joint Conf. on
  Artificial Intelligence, {IJCAI} 2018}, pages 5103--5111, 2018.

\bibitem[\protect\citeauthoryear{Shih \bgroup \em et al.\egroup
  }{2019}]{Shih_19}
Andy Shih, Arthur Choi, and Adnan Darwiche.
\newblock Compiling bayesian network classifiers into decision graphs.
\newblock In {\em The Thirty-Third {AAAI} Conf. on Artificial Intelligence,
  {AAAI} 2019}, pages 7966--7974, 2019.

\bibitem[\protect\citeauthoryear{Sokol and Flach}{2020}]{Sokol_20}
Kacper Sokol and Peter~A. Flach.
\newblock Explainability fact sheets: a framework for systematic assessment of
  explainable approaches.
\newblock In {\em FAT* '20: Conference on Fairness, Accountability, and
  Transparency, Barcelona, Spain, January 27-30, 2020}, pages 56--67, 2020.

\bibitem[\protect\citeauthoryear{Teze \bgroup \em et al.\egroup
  }{2018}]{Teze_18}
Juan~Carlos Teze, Lluis Godo, and Guillermo~Ricardo Simari.
\newblock An argumentative recommendation approach based on contextual aspects.
\newblock In {\em Scalable Uncertainty Management - 12th Int. Conf., {SUM}
  2018}, pages 405--412, 2018.

\bibitem[\protect\citeauthoryear{Timmer \bgroup \em et al.\egroup
  }{2015}]{Timmer_15}
Sjoerd~T. Timmer, John{-}Jules~Ch. Meyer, Henry Prakken, Silja Renooij, and
  Bart Verheij.
\newblock Explaining bayesian networks using argumentation.
\newblock In {\em Symbolic and Quantitative Approaches to Reasoning with
  Uncertainty - 13th European Conf., {ECSQARU} 2015}, pages 83--92, 2015.

\bibitem[\protect\citeauthoryear{UCI}{2020}]{UCIDatasets}
{Center for Machine Learning and Intelligent Systems} UCI.
\newblock {Machine Learning Repository}, 2020.

\bibitem[\protect\citeauthoryear{White and d'Avila Garcez}{2020}]{White_20}
Adam White and Artur~S. d'Avila Garcez.
\newblock Measurable counterfactual local explanations for any classifier.
\newblock In {\em 24th European Conference on Artificial Intelligence ({ECAI}
  2020)}, 2020.

\bibitem[\protect\citeauthoryear{Zeng \bgroup \em et al.\egroup
  }{2018}]{Zeng_18}
Zhiwei Zeng, Xiuyi Fan, Chunyan Miao, Cyril Leung, Jing~Jih Chin, and
  Yew{-}Soon Ong.
\newblock Context-based and explainable decision making with argumentation.
\newblock In {\em Proceedings of the 17th Int. Conf. on Autonomous Agents and
  MultiAgent Systems, {AAMAS} 2018}, pages 1114--1122, 2018.

\end{thebibliography}


\section*{Appendix A: Theoretical analysis proofs}

\textbf{Proposition \ref{thm:influences}.}
\begin{proof}
    If $\Influences = \Influences_{io}$ then from Defs. 2 and 3, $\{ (x,c) \in \Args \times \Cla | (c,x) \in \Dep \} = \Obs \times \Cla_o$ and it follows that $\Dep = \Cla_o \times \Obs$.
    If $\Dep = \Cla_o \times \Obs$ then from Defs. 2 and 3, $\Influences = \Obs \times \Cla_o = \Influences_{io}$.
\end{proof}

\noindent\textbf{Proposition \ref{thm:DIDXtoSDIDX}.}
\begin{proof}
For any $(x,y) \in \Influences$ and $a \in \Inputs$, by Def. 7, if $\prop_-((x,y), a) = true$ then 
$\forall x_k \in \Val(x) \setminus \{ \SF(a, x) \},
        P(\SF(a, y) | a)
        <
        P(\SF(a, y) | a'_{x_k})$.
        
Then we get:
$\frac{
        \sum\limits_{x_k \in \Val(x) \setminus \{ \SF(a, x) \}} \left[ P(x_k) \cdot
        P(\SF(a, y) | a'_{x_k}) \right]
        }{
        \sum\limits_{x_k \in \Val(x) \setminus \{ \SF(a, x) \}}  P(x_k) 
        } \geq 
        \frac{
        \sum\limits_{x_k \in \Val(x) \setminus \{ \SF(a, x) \}} \left[ P(x_k) \cdot
       \min_{x_k \in \Val(x) \setminus \{ \SF(a, x) \}} P(\SF(a, y) | a'_{x_k}) \right]
        }{
        \sum\limits_{x_k \in \Val(x) \setminus \{ \SF(a, x) \}}  P(x_k) }    =   \min_{x_k \in \Val(x) \setminus \{ \SF(a, x) \}} P(\SF(a, y) | a'_{x_k})   > P(\SF(a, y) | a)$.

Then by Def. 8, $\prop_{\minusSTAT}((x,y), a) = true$.
The proof for $\Supps \subseteq \SDSupps$ is analogous.
It follows from Def. 7 that $\Args_r \subseteq \Args'_r$.
\end{proof}

\noindent\textbf{Proposition \ref{thm:equivalence}.}
\begin{proof}
    If for any $x \in \Args'_r \setminus \{ e \}$, $|\Val(x)| = 2$ then from Defs. 7 and 8 $\frac{
        \sum\limits_{x_k \in \Val(x) \setminus \{ \SF(a, x) \}} \left[ P(x_k) \cdot P(\SF(a, y) | a'_{x_k}) \right]}
        {\sum\limits_{x_k \in \Val(x) \setminus \{ \SF(a, x) \}}  P(x_k)  } 
        = 
        \frac{P(x_k) \cdot P(\SF(a, y) | a'_{x_k})}
        {P(x_k)}
        =
        P(\SF(a, y) | a'_{x_k})$, where $x_k$ is the only element of $\Val(x) \setminus \{ \SF(a, x) \}$. 
        Thus, $\prop_{-} = \prop_{\minusSTAT}$ and $\prop_{+} = \prop_{\plusSTAT}$, giving $\Atts = \SDAtts$ and $\Supps = \SDSupps$, \resp\ It follows from Def. 7 that $\Args_r = \Args'_r$.
\end{proof}

\noindent\textbf{Proposition \ref{thm:monotonicity}.}
\begin{proof}
    \AR{Monotonically dialectical explanation kits are} 
    monotonic by inspection of Definition 7. \AR{Stochastically dialectical, LIME and SHAP explanation kits} \EA{
    are not monotonic as proved by the results in Table 6.}
\end{proof}

\noindent\textbf{Proposition \ref{thm:CFXvsIDX}.}
\begin{proof}
For any $(x,y) \in \Influences$ and $a \in \Inputs$, by Def. 9, if $\prop_!((x,y), a) = true$ then 
$\forall a' \in \Inputs$ such that $\SF(a', \arga) \neq \SF(a, \arga)$ and $\forall \argc \in \Influences(\argb) \backslash \{ \arga \}$ $\SF(a', \argc) = \SF(a, \argc)$,
it holds that $\SF(a, \argb) \neq \SF(a', \argb)$.
Thus, it must be the case that 
$\forall x_k \in \Val(x) \setminus \{ \SF(a, x) \},
        P(\SF(a, y) | a)
        >
        P(\SF(a, y) | a'_{x_k})$
and so, by Def. 7, $\prop_+((x,y), a) = true$.
\end{proof}

\section*{Appendix B: Empirical Experiments Setup}

\subsection*{Datasets and dataset splits}
As reported is Table \ref{table:datasets_details} we ran experiments in several settings. We divided the datasets in a randomly stratified train and test sets with a split 75/25\% split. When the data source was a Bayesian network we artificially generated the corresponding combinatorial dataset using variable elimination (an exact inference algorithm for Bayesian network).

\begin{table}[t]
	\begin{small}
		\begin{center}
		\begin{tabular}{r|cc|cc|cc|cc}
Setting & Data Source$^{*}$  & Classifier Implementation$^{\dagger}$ \\
\hline
\textbf{Shuttle} & DT & \code{skl}\\
\textbf{Parole} & DT & \code{skl}\\
\textbf{German} & DT & \code{skl}\\
\textbf{COMPAS} & DT & \code{skl}\\
\textbf{Car} & DT & \code{skl}\\
\textbf{Emotions} & DT & Chain of \code{skl}\\
\textbf{Asia} & BN & Chain of \code{pgm}\\
\textbf{Child} & BN & Chain of \code{pgm}\\
\end{tabular}
		\end{center}
    	\protect\caption{Experimental settings details
    	($*$) \textbf{D}a\textbf{T}aset (DT) or \textbf{B}ayesian \textbf{N}etwork (BN); ($\dagger$) \code{sklearn}.\code{CategoricalNB} (\code{skl}) or \code{pgmpy}.\code{BayesianModel} (\code{pgp});  }\label{table:datasets_details}
    \end{small}
\end{table}

Note that the experiments on the explanations were run on the test set, and if the number of samples in the test set was greater than 250 samples, we ran them on a random sample of 250 samples.

 
 \subsection*{Execution details}
LIME explanations were generated using default parameters. 
SHAP explanations were generated using default parameters.
For all the random computations (\code{sklearn}, \code{pandas}, \code{numpy}, \code{random}) we used a random seed of $0$.


\subsection*{Software Platform and Computing Infrastructure}
To run our experiments, we used a machine with Ubuntu 18.04, a single Intel i9-9900X processor, 32GB of RAM and no GPU acceleration. We used a Python 3.6 environment with {networkx} 2.2, pgmpy 0.1.11, {scikit-learn} 0.22.1, {shap} 0.35.0 and {lime} 0.1.1.37 as software platform to run our experiments.

\subsection*{Hyper-parameter training}
The classifier \code{sklearn} \code{CategoricalNB} has 2 hyper-parameters: $\alpha \in \mathbb{R}$ (Laplace smoothing coefficient) and a tuple of probabilities $\beta = \langle \beta_1, \ldots \beta_m \rangle$ (classes prior probabilities) such that $\beta_i \in [0, 1]$ and $m$ is the number of classes. 
\code{pgmpy} \code{BayesianModel} has no learning parameters (all probabilities are given in the Bayesian network, they do not need to be learnt).
Datasets with numeric features have more hyper-parameters: the type of bins to make numeric variables categorical, that we denote as $\gamma \in \Gamma = \{  SS, SL, custom \}$ where $SS$ denotes Same Size (i.e. each bucket has the same number of elements from the dataset), $SL$ denotes Same Length (i.e. each bucket covers an interval of the same length) and $custom$ denotes a custom discretization, the number of buckets used in the discretization, denoted with $\delta_i$ for each numeric feature $i$.

\textbf{Shuttle}. $\alpha = 5$, $\beta = \langle 0.8, 0.2 \rangle$.

\textbf{Votes}. $\alpha = 1$, $\beta = \langle 0.4, 0.6 \rangle$.

\textbf{Parole}.$\alpha = 0.1$, $\beta = \langle 0.9, 0.1 \rangle$, $\gamma = SL$, $\delta_{time served} = 2$, $\delta_{max sentence} = 11$, $\delta_{age} = 9$.

\textbf{German}. $\alpha = 0.00001$, $\beta = \langle 0.35, 0.65 \rangle$, $\gamma = SL$, $\delta_{age} = 10$, $\delta_{amount} = 10$, $\delta_{duration} = 9$.

\textbf{COMPAS}. $\alpha = 0.1$, $\beta = auto$, $\gamma = custom$. We used a custom discretization for the dataset obtained with manual trials because it was performing better than automatic ones (same size or same length bins).

\textbf{Car}. $\alpha = 1$, $\beta = \langle 0.2, 0.1, 0.6, 0.1 \rangle$.

\textbf{Emotions}. $\alpha = 0.00001$, $\beta = \langle 0.6, 0.4 \rangle$, $\gamma = SS$, $\delta_i = 9$ for all the features. 
We built the classes tree using the mutual information coefficient as described in \cite{ChainClassifiers}, but instead of creating an ensemble of chain classifiers we considered only a single one and we considered the root of the tree as an hyper-parameter, (\code{amazed-lonely} was deemed as root in our case).

\end{document}